\documentclass[twoside,11pt]{article}

\usepackage[preprint]{jmlr2e_nodeco}
\usepackage{mathtools}
\mathtoolsset{showonlyrefs}
\usepackage{amsmath,amssymb,amsthm}
\usepackage{bm}
\usepackage{dsfont}
\usepackage{enumitem}
\usepackage[final]{showlabels} %
\usepackage[table,xcdraw]{xcolor}
\usepackage{multirow}
\usepackage{circuitikz}
\usepackage{hhline}
\usepackage{booktabs} %
\usepackage{hyperref}
\usepackage[mathscr]{eucal}    %
\usepackage{centernot}
\usepackage{multirow}
\usepackage{hhline}
\usepackage{multicol}

\newtheorem{theoremA}{Theorem}[section]
\newtheorem{lemmaA}[theoremA]{Lemma}

\definecolor{Brown}{rgb}{0.64,0.16,0.16}
\definecolor{OliveGreen}{rgb}{0.1,0.4,0.1}

\newcommand{\R}{\mathbb{R}}   %
\renewcommand{\P}{\mathbb P} %
\newcommand{\Q}{\mathbb Q} %
\newcommand{\E}{\mathbb{E}} %
\renewcommand{\d}{\mathrm{d}} %
\newcommand{\tb}{\textbf}    %
\newcommand{\fip}[2]{\left\langle{#1}\right\rangle_{#2}} %
\newcommand{\norm}[2]{\left\|{#1}\right\|_{#2}} %
\newcommand{\KL}{\operatorname{KL}} %
\newcommand{\supp}{\operatorname{supp}} %
\newcommand{\MMD}[1][K]{\operatorname{MMD}_{#1}}
\newcommand{\HSIC}[1][K]{\operatorname{HSIC}_{#1}}
\newcommand{\KSD}[1][K_{\P_0}]{\operatorname{KSD}_{#1}}
\renewcommand{\H}{\mathscr{H}} %
\renewcommand{\O}{\mathcal{O}} %
\newcommand{\X}{\mathcal{X}} %
\newcommand{\Y}{\mathcal{Y}} %
\newcommand{\Np}{\mathbb{N}_{>0}} %
\newcommand{\N}{\mathbb{N}_{0}} %
\renewcommand{\d}{\mathrm{d}} %
\newcommand{\M}{\mathcal{M}}
\newcommand{\jover}[2]{\overset{\text{(#1)}}{#2}} %
\newcommand{\marg}[2]{#1|_#2} %

\newcommand{\sfm}[1][\P_0]{\Psi_{#1}} %

\newcommand{\mP}{\mathcal{P}_1}
\newcommand{\mC}{\mathcal C}

\newtheorem{theorem}{Theorem}
\newtheorem{lemma}[theorem]{Lemma}

\newtheorem{remark}[theorem]{Remark}
\newtheorem{corollary}[theorem]{Corollary}

\newtheorem{assumption}[theorem]{Assumption}

\usepackage{lastpage}
\jmlrheading{xx}{2026}{1-\pageref{LastPage}}{7/26; Revised xx}{xx}{21-0000}{Jose Cribeiro-Ramallo and Florian Kalinke and Zoltán Szabó}

\ShortHeadings{Minimax Lower Bound of Kernel Discrepancies}{Cribeiro-Ramallo, Kalinke, and Szabó}
\firstpageno{1}

\begin{document}
\renewcommand{\theequation}{\thesection.\arabic{equation}} %

\title{Minimax Lower Bounds of Kernel Discrepancy Estimation: MMD, HSIC, KSD}

\author{\name Jose Cribeiro-Ramallo \email jose.cribeiro@kit.edu \\
       \addr Chair of Information Systems\\
       Karlsruhe Institute of Technology\\
       Am Fasanengarten 5, 76131 Karlsruhe, Germany
       \AND
       \name Florian Kalinke \email florian.kalinke@kit.edu\\
      \addr Chair of Information Systems\\
       Karlsruhe Institute of Technology\\
       Am Fasanengarten 5, 76131 Karlsruhe, Germany
       \AND
       \name Zoltán Szabó \email z.szabo@lse.ac.uk \\
       \addr Department of Statistics\\ London School of Economics\\
       Houghton Street, London, WC2A 2AE, UK}

\editor{My editor}

\maketitle

\begin{abstract}%
Over the past 20 years, kernel discrepancies have been leveraged as a highly powerful tool for quantifying the disagreement of distributions, with numerous successful applications in two-sample, goodness-of-fit, and independence testing, among others.
Their fastest estimators are known to converge at a parametric rate---$n^{-1/2}$---under mild conditions. While this rate is known to be minimax optimal on $\mathbb R^d$ under strict assumptions with bounded kernels, little is known about its optimality beyond the finite-dimensional Euclidean setting with unbounded kernels.
In this work, we prove that the minimax lower bound of estimation of the most popular kernel discrepancies (maximum mean discrepancy, Hilbert-Schmidt independence criterion and kernel Stein discrepancy; MMD, HSIC, KSD) is $n^{-1/2}$ on general topological spaces, and under mild assumptions on the kernel; the same rates are shown (as corollaries) to hold for the estimation of the mean embedding and the centered cross-covariance operator.
Our results settle the question of optimal estimation of these kernel discrepancies.
\end{abstract}

\begin{keywords}
maximum mean discrepancy, Hilbert-Schmidt independence criterion, kernel Stein discrepancy, minimax lower bound, reproducing kernel Hilbert space, perturbations
\end{keywords}

\section{Introduction} \label{sec:intro}
Kernel discrepancies appeared over 20 years ago in order to tackle the problem of quantifying the similarity of distributions, relying on the representation power and computational tractability of reproducing kernel Hilbert spaces (RKHS; \citealt{aronszajn50theory}). Kernel functions~\citep{wahba90spline,atteia92hilbertian,steinwart08support,saitoh16theory} have been designed and studied on a wide variety of domains, rendering these discrepancies broadly applicable; examples include sequences~\citep{watkins99dynamic2,lodhi02text,kiraly19kernel,chevyrev22signature, toth25random}, probability distributions~\citep{berlinet04reproducing,hein05hilbertian,smola07hilbert,sriperumbudur10hilbert,marienwald21high}, sets~\citep{haussler99convolution,gartner02multi}, rankings~\citep{jiao16kendall, mania18kernel}, fuzzy domains~\citep{guevara17cross}, graphs~\citep{kondor02diffusion,borgward05shortestpath,kondor16multiscale2,borgwardt20graph,schulz22graph,nikolentzos23graph}, and Abelian groups~\citep{steinwart21strictly}.

Most kernel discrepancies rely on the kernel mean embedding \citep{berlinet04reproducing,smola07hilbert}, a lossless (under mild conditions) map of Borel probability measures into an RKHS. This lossless property holds by definition if and only if (iff) the mean embedding is injective; in this case the kernel is called characteristic~\citep{fukumizu08kernel2,sriperumbudur10hilbert}.
Considering the RKHS distance between two mean embeddings results in the so-called maximum mean discrepancy (MMD; \citealt{gretton12kernel}),
known to be equivalent \citep{sejdinovic13equivalence} to energy distance \citep{baringhaus04new,szekely04testing,szekely05new}---also known as $N$-distance; \citep{zinger92characterization,klebanov05ndistance}. It is also a specific instance of integral probability metrics (IPM; \citealt{zolotarev83probability,muller97integral}). MMD is a metric on the space of Borel probability measures provided that the kernel is characteristic.  If  MMD is applied to a product space---with a product kernel---to quantify the discrepancy of a  joint distribution and the product of its marginals, one gets the Hilbert-Schmidt independence criterion (HSIC). HSIC was originally designed for $M=2$ components \citep{gretton05measuring,gretton05kernel}, and later extended to $M\geq2$ \citep{quadrianto09kernelized2,sejdinovic13kernel2,pfister18kernel}. HSIC is a valid independence measure for $M=2$ if the kernel components are all characteristic \citep{lyons13distance}; requiring $c_0$-universality suffices for $M>2$ \citep{szabo18characteristic2}.
When applying MMD to a target and a sampling distribution, and selecting the kernel such that the mean embedding of the target vanishes (via a Stein operator; \citealt{stein72bound,chen21stein,anastasiou23stein}), one obtains the kernel Stein discrepancy (KSD; \citealt{chwialkowski16kernel,liu16kernelized}). KSD is a valid goodness-of-fit (GoF) measure given that the mean embedding of the target distribution is unique; $c_0$-universality~\citep{carmeli10vector,sriperumbudur10hilbert} of the kernel is known to be a sufficient condition for this property on $\R^d$~\citep[Theorem~2.2]{chwialkowski16kernel}.

The computational tractability and broad applicability of kernel discrepancies led to their widespread use and their extensions to other domains. Applications include, in case of MMD,
testing~\citep{gretton12kernel,fernandez19maximum,liu20learning,wynne22kernel,chevyrev22signature,schrab23mmd,zhou25dual,chatalic26scalable}, generative modeling~\citep{dziugaite15mmdgan,li15generative,sutherland17generative,li17mmd,zhou25inductive}, quantization~\citep{arbel19maximum,gladin24interaction,belhadji26weighted},
subspace selection~\citep{cribeiro25adversarial},
robust regression~\citep{alquier24universal},
learning world models~\citep{balestriero25lejepa,zimmermann25kerjepa},
copula estimation~\citep{alquier23estimation},
model selection~\citep{brueck25distribution},
and online change detection~\citep{li15mstatistic,li19scanbstatistics,kalinke25maximum,kalinke25optimal}; HSIC sees use for
independence testing~\citep{gretton08kernel2,chwialkowski14kernel,wehbe15nonparametric,bilodeau17tests,pfister18kernel,albert22adaptive,shekhar23permutation,podkopaev23sequential,ren24efficiently,ren25regression,zhou25dual},
feature selection~\citep{song12feature},
object detection~\citep{zhang26robust},
representation learning~\citep{miklautz25hsplid},
clustering~\citep{song07dependence,gonzalez19block},
and causal discovery~\citep{mooij16distinguishing,chakraborty19distance,kalinke23nystrommhsic}; lastly, KSD has been popularized in GoF testing~\citep{liu16kernelized,chwialkowski16kernel,fernandez20kernelized,schrab22ksdagg,baum23kernel,martinez25sequential,kalinke25nystromksd,kalinke26nystromksdjmlr,hagrass26stein},
model comparison~\citep{lim19kernel,kanagawa23kernel},
distribution compression~\citep{li24debiased},
model validation~\citep{gorham17measuring,futami19bayesian,hodgkinson21reproducing,wang23stein},
test validation~\citep{liu25robustness},
learning world models~\citep{zimmermann25kerjepa},
learning variational models~\citep{liu16stein,liu18stein,chen18stein,chen19stein,korba20svgd,korba21ksddescent},
and model explainability~\citep{sarvmaili25explaining}. The flexibility of kernel functions allows these KSDs to be applied on a wide variety of domains; examples include discrete spaces \citep{yang18discrete}, Riemannian manifolds \citep{xu20stein,xu21interpretable,barp22riemannstein,cheng24kernel}, Hilbert spaces \citep{wynne22kernel,zhang24fast,wynne24spectral,wynne25statistical}, point-processes \citep{chwialkowski14kernel,yang19point,kidger19deep,laumann21kernel,chevyrev22signature}, and graph data \citep{xu21gofgraph}.

Many estimators for these discrepancies exist; classical ones are based on U- and V-statistics~\citep{gretton05kernel,quadrianto09kernelized2,gretton12kernel,liu16kernelized,chwialkowski16kernel,pfister18kernel}; those for HSIC and KSD are known to converge at a rate of $n^{-1/2}$, while the rate is $n^{-1/2} +m^{-1/2}$ in the case of the MMD.
These convergence rates are achieved in any topological space with bounded~\citep{smola07hilbert} and unbounded kernels (under subexponentiality; \citealt{kalinke25nystromksd,kalinke26nystromksdjmlr}).
Further, several accelerated versions exists with the same $n^{-1/2}$-rate~\citep{chatalic22nystrom,kalinke23nystrommhsic,kalinke25nystromksd}.

Lower bounds for MMD and HSIC estimation with the same $(n^{-1/2} + m^{-1/2})$- and $n^{-1/2}$-rates are only known on $\R^d$ for radial and translation-invariant kernels \citep{tolstikhin16minimax2,kalinke24minimax} and on $\R$ ($d=1$) for the exponential kernel \citep{chamakh24keep}, and are not trivially extendable beyond the finite-dimensional Euclidean setting. The question of whether there exist matching lower bounds %
beyond $\R^d$ under mild assumptions on the kernels remains open. This question, to which we give an affirmative answer, is the main focus of our manuscript. In particular, we make the following \tb{contributions}:
\begin{enumerate}[label=(\roman*)]
    \item We establish the minimax lower bound of $n^{-1/2} + m^{-1/2}$, $n^{-1/2}$, and $n^{-1/2}$ for the estimation of MMD, HSIC, and KSD (respectively) on general topological spaces with unbounded kernels. These lower bounds match the upper ones of available estimators, hence settle their minimax optimality.
    \item Our results imply the same minimax lower bound of $n^{-1/2}$ for the estimation of the mean embedding and the centered cross-covariance operator.
\end{enumerate}
This article expands the work of \citet{cribeiro26minimax}, which settled the minimax lower bound of KSD estimation on general topological spaces (recalled in Theorem~\ref{th:ksd}, with proof for completeness), to MMD~(Theorem~\ref{th:mmd}), mean embedding~(Corollary~\ref{cor:mean_emb}), HSIC~(Theorem~\ref{th:hsic}), and centered cross-covariance operator estimation (Corollary~\ref{cor:cross_cov}), under similar mild conditions.

The paper is structured as follows. Notations are introduced in Section~\ref{sec:notations}. Section~\ref{sec:results} is dedicated to our main results on the minimax lower bound of MMD, HSIC and KSD estimation on topological spaces, with corollaries on the estimation of the mean embedding and the centered cross-covariance operator. Our proofs are gathered in Appendix~\ref{sec:proofs}, with auxiliary and external results collected in Appendix~\ref{sec:aux} and Appendix~\ref{sec:ext}, respectively.

\section{Notations}\label{sec:notations}

In this section we introduce the following notations: $\N$, $\Np$, $[d]$, $\times_{i=1}^d\X_i$, $\X^d$, $\overline{S}$, $f(A)$, $f^{-1}(B)$, $\mC(\X)$, $\mC_b(\X)$, $\mathcal L_1(\X,\nu)$,
$\mathcal O(b_n)$, $\Omega(b_n)$, $\Theta(b_n)$, $\mathcal B(\X)$, $\M_b(\X)$, $\M_1^+(\X)$, $\supp(\P)$, $\delta_{x}$, $X_{1:n}$, $\P^n$, $\tfrac{\d\Q}{\d\P}$, $\KL(\cdot\|\cdot)$, $\otimes_{i=1}^d\P_i$,  $\O_\P(a_n)$, $\E_\P$, $B(\H)$, $K$, $K(\cdot,x)$, $\H_K$, $\norm{\cdot}{\H_K}$, $\mu_K(\cdot)$, $\mP(K;\X)$, $\MMD(\cdot,\cdot)$, $\marg\P i$, $\otimes_{i=1}^d\H_{K_i}$, $\otimes_{i=1}^dK_i$, $\otimes_{i=1}^d f_i$, $\HSIC(\cdot)$, $C_K(\cdot)$, $\sfm(\cdot)$, $T_{\P_0}(\cdot)$, $K_{\P_0}$, $\mathcal T$, $\KSD(\cdot,\cdot)$.

\paragraph{General:}
Denote by $\N=\{0,1,2,\dots\}$ the set of  natural numbers. For a positive integer $d\in\Np=\{1,2,\dots\}$, let $[d]=\{1,2,\dots,d\}$. For a sequence of sets $(\X_i)_{i=1}^d$, we denote their Cartesian product by $\times_{i=1}^d\X_i$; when $\X_i=\X$ for all $i\in[d]$, we write $\X^d$. Let $(\X,\tau_{\X})$ be a topological space. For a set $S\subseteq\X$ its closure $\overline S$ is defined as the smallest closed subset of $\X$ containing $S$. A set $S\subseteq\X$ is called dense in $\X$ if $\overline S = \X$;
$\X$ is called separable if it has a countable dense subset. Given a map $f:\X\to\Y$, the image of $A\subseteq \X$ is $f(A)=\{y\in\Y:\,y=f(x),\, x\in A \}$; the pre-image of $B\subseteq \Y$ is $f^{-1}(B) = \{x\in\X:\, f(x) \in B\}$.
The spaces of continuous and continuous bounded real-valued functions on $\X$ are denoted by $\mC(\X)$ and $\mC_b(\X)$, respectively.
For a measure space $(\X,\Sigma,\nu)$, $\mathcal L_1(\X,\nu)$ is the space of real-valued functions on $\X$ which are finitely integrable with respect to the measure $\nu$.
For positive sequences $(a_n)_{n=1}^\infty$ and  $\left(b_n\right)_{n=1}^\infty$, (i) $a_n = \O(b_n)$ if there exist $C>0$ and $n_0\in\Np$ such that $a_n \le C b_n$ for all $n\ge n_0$,  (ii) $a_n = \Omega(b_n)$ if $b_n = \O(a_n)$, an (iii) $a_n = \Theta(b_n)$ if $a_n = \O(b_n)$ and $b_n = \O(a_n)$.

\paragraph{Probability:} Let $(\X, \tau_{\X})$ be a topological space, which we assume throughout the paper to be enriched with its Borel sigma-algebra $\mathcal B(\X)\coloneq\mathcal B(\X,\tau_{\X})$; the space of finite signed (resp.\ probability) measures on $(\X,\mathcal B(\X))$ is denoted by $\M_b(\X)$ (resp.\ $\M_1^+(\X)$).
For $\P\in\M_1^+(\X)$, its support $\supp(\P)$ is defined as the smallest closed set whose complement has null $\P$-measure: a closed set $C=\supp(\P)$ if (i) $\P(\X\setminus C)=0$ and (ii) for any closed set $C_1\subseteq \X$ for which $\P(\X\setminus C_1)=0$, $C\subseteq C_1$ holds.\footnote{Note that the support of a Borel measure exists \citep[Proposition 2.3]{kozarzewski18existence} if the topology of $\X$ has a countable basis (in other words, $\X$ is second countable); in metric spaces, the notion of second countability and separability coincide \citep[Proposition 2.1.4]{dudley04real}. \label{foot:2ndcount}}
For a point $x\in \X$, the Dirac measure $\delta_x\in\M_1^+(\X)$ is defined as
\begin{align}
    \delta_x(A) =\begin{cases} 1 &\text{ if } x\in A,\\
    0&\text{ if } x\not \in A,\end{cases} \quad \text{ for }  A\in\mathcal B(\X).
\end{align}
An independent and identically distributed (i.i.d.) sample from $\P \in\mathcal M_1^+(\X)$ is denoted by $X_{1:n} = (X_1,\dots,X_n)$ (shortly, $X_{1:n}\sim \P^n$, with $\P^n$ standing for the $n$-fold product of $\P$). Let $\Q,\P\in\mathcal M_{1}^+(\X)$, and let $\Q$ be absolutely continuous w.r.t.\ $\P$ ($\Q\ll \P$, with Radon-Nikodym derivative denoted by $\tfrac{\d \Q}{\d \P}$); the Kullback–Leibler divergence of $\Q$ and $\P$ is defined as $\KL(\Q\|\P) =\int_{\X}\ln \!\left(\tfrac{\d \Q}{\d \P}(x)\right)\d \Q(x)$.
Given topological spaces $(\X_i,\tau_{\X_i})_{i=1}^{d}$, their product $\times_{i=1}^d \X_i$ is assumed
throughout the manuscript to be endowed with the product topology; when $[d]$ is finite (as in our case) this topology is known \citep[Proposition~10.10]{sutherland09introduction} to coincide with the box topology $\tau = \{\times_{i=1}^dU_i\subseteq \times_{i=1}^d\X_i:~ U_i\in\tau_{\X_i}\}$.
The product of probability measures $\P_i\in\M_1^+(\X_i)$ ($i\in[d]$) is denoted by $\otimes_{i=1}^d\P_i\in\M_1^+(\times_{i=1}^d\X_i)$. For a sequence of i.i.d.\ real-valued random variables $(X_n)_{n=1}^\infty$, each having law $\P$, and a sequence of positive reals $(a_n)_{n=1}^\infty$, $X_n=\O_\P(a_n)$ means that $\left(\frac{X_n}{a_n}\right)_{n=1}^\infty$ is bounded in probability. Let $\H$ be a Hilbert space and $X \sim \P \in \mathcal M_{1}^+(\H)$. If $\int_\H \norm{x}{\H}\d\P(x) < \infty$, the expectation of $X$ is $\E_\P[X] \ =\ \int_{\H} x\d \P(x)$, where the integral is meant in Bochner's sense.

\paragraph{(Reproducing kernel) Hilbert space:}
Let the unit ball of a Hilbert space $\H$ be denoted by $B(\H)=\{h\in\H\,:\,\left\|h\right\|_\H \le 1\}$.
A function $K:\X^2\to\R$ is called a kernel if there exists a Hilbert space $\H$ and a feature map $\Phi:\X\to\H$ such that $K(x,x') = \fip{\Phi(x),\Phi(x')}{\H}$ for all $x,x'\in\X$. A Hilbert space $\H_K$ of $\X \to \R$ functions is called reproducing kernel Hilbert space (RKHS) associated to a kernel $K:\X^2\to \R$ if $K(\cdot,x)\in \H_{K}$ and $\fip{K(\cdot,x),f}{\H_{K}}=f(x)$ for all $x\in\X$ and $f\in\H_K$; $K(\cdot,x)$ denotes the function $x'\mapsto K(x',x)$ ($x\in\X$) and it is called the canonical feature map (of $K$); $K$ is called the reproducing kernel of $\H_K$. A kernel is equivalent to a reproducing kernel, and the correspondence of kernels and RKHSs is one to one. The norm in $\H_K$ is denoted by $\norm{f}{\H_K} = \sqrt{\fip{f,f}{\H_K}}$ ($f \in \H_K$).

\paragraph{Maximum mean discrepancy:}
Let $K:\X^2\to\R$ be a Borel-measurable kernel; we consider all kernels to be Borel-measurable in the manuscript. The mean embedding w.r.t.\ $K$ is defined as
\begin{align}
    \mu_K: \mP(K;\X) &\to \H_K, && \P \mapsto  \int_\X K(\cdot, x)\d \P(x),
\end{align}
with
\begin{align}
\mP(K;\X) &= \Bigg\{\P\in \M_1^+(\X):~\int_\X \norm{K(\cdot,x)}{\H_K}\d\P(x) =\int_\X \sqrt{K(x,x)} \d \P(x)<\infty\Bigg\}\subseteq\M_1^+(\X).
\end{align}
$K$ is called characteristic to $\mP(K;\X)$ if $\mu_K$ is injective (\citealt[page~2270]{sejdinovic13equivalence}; \citealt[Definition~1]{simon-gabriel18kernel}).\footnote{
 If $K$ is bounded (in other words, $\sup_{x,x'\in \X}K(x,x') <\infty$), one has that $\mP(K;\X)=\M_1^+(\X)$, and hence one gets back the (classical) characteristic property of kernels \citep{fukumizu08kernel2,sriperumbudur10hilbert}. Continuous characteristic kernels are known \citep[Corollary~3.18]{steinwart21strictly} to exist on arbitrary compact metrizable topological space.}
 The maximum mean discrepancy (MMD) of $\P,\Q\in \mP(K;\X)$ is defined as \begin{align}
 \MMD(\P,\Q) =\norm{\mu_K(\P)-\mu_K(\Q)}{\H_K}.
 \end{align}
\paragraph{Hilbert-Schmidt independence criterion:} For $i\in[d]$, let  $(\X_i,\tau_{\X_i})$ be topological spaces and, for the remainder of this section, $\X = \times_{i=1}^d \X_i$.
For a probability measure $\P\in\M_1^+(\X)$, its $i$-th marginal $\marg{\P}{i}\in\M_1^+(\X_i)$ ($i\in[d]$) is
 \begin{align*}
\P|_i(A_i) & =  \int_{\X_1\times\cdots \times \X_{i-1}\times A_i\times \X_{i+1}\times \cdots \times \X_d}1\ \d\P(x) \text{ for any } A_i\in\mathcal B(\X_i).
 \end{align*}
Assume that each topological space $\X_i$ is equipped with a kernel $K_i:\X_i^2\to\R$ with associated RKHS $\H_{K_i}$ ($i\in[d])$. The tensor product of $(\H_{K_i})_{i=1}^d$ is denoted by $\otimes_{i=1}^d\H_{K_i}$; it is an RKHS with the product kernel $K=\otimes_{i=1}^d K_i:\X^2\to\R$ defined as $ K(x,y) = \prod_{i=1}^dK_i(x_i,y_i)$ for all $x=(x_i)_{i=1}^d\in\X$ and $y=(y_i)_{i=1}^d\in\X$ \citep[Theorem~13]{berlinet04reproducing}. The kernel $K$ has the canonical feature map $K(\cdot,(x_i)_{i=1}^d)= \otimes_{i=1}^dK_{i}(\cdot,x_i)\in\otimes_{i=1}^d\H_{K_i}=\H_K$, where $\otimes_{i=1}^df_i: \X\to\R$ denotes the tensor product of $f_i\in \H_{K_i}$, defined as $(x_i)_{i=1}^d\mapsto \prod_{i=1}^d f_i(x_i)$.

The Hilbert-Schmidt independence criterion (HSIC) for $K = \otimes_{i=1}^d K_i$ on $\X$ and for $\P \in \mP(K; \X)$ is defined as
\begin{align}
\HSIC(\P) &= \MMD\!\big(\P,\otimes_{i=1}^d\marg\P i) = \norm{C_K(\P)}{\H_K}, \label{eq:hsic-definition}\\
C_K(\P) & =\mu_K(\P) - \mu_K\big(\otimes_{i=1}^d\marg\P i\big)=\mu_K(\P) - \otimes_{i=1}^d \mu_{K_i}(\marg\P i)\in\H_K \label{eq:cov-op-definition},
\end{align}
where $C_{K}(\P)$ is the centered cross-covariance operator. $K$ is called $\mathcal I$-characteristic (characteristic to independence; \citealt{szabo18characteristic2}) if $\HSIC(\P)=0$ iff $\P=\otimes_{i=1}^d \marg{\P}{i}$ holds for all $\P\in \mP(K;\X)$.

\paragraph{Kernel Stein discrepancy:}
Let $(\X,\tau_\X)$ be a topological space, $\P_0\in\M_1^+(\X)$, $\H$ a Hilbert space of functions on $\X$, and $\sfm:\X\to\H$ a measurable map satisfying   \begin{align}
\E_{\P_0} [\sfm(X)] &= 0. \label{eq:sfm=0}
\end{align}
One can define the \citeauthor{stein72bound} operator $T_{\P_0}$ on $\H$ as $(T_{\P_0}f)(x) = \fip{\sfm(x),f}{\H}$, with $f\in\H$ and $x\in\X$; the operator inherits the mean-zero property
\begin{align}
	\E_{\P_0}\!\left[ \left( T_{\P_0} f\right)(X) \right] &= \fip{\E_{\P_0} [\sfm(X)],f}{\H} = 0,
\end{align}
which can be seen by interchanging the inner product with the expectation and  using that $\E_{\P_0}[\sfm(X)] = 0$.
We define the Stein kernel associated to $\sfm$ as
\begin{align}
K_{\P_0}(x,x') = \fip{\sfm(x), \sfm(x')}{\H}
\end{align}
for all $(x,x')\in\X^2$. As $K_{\P_0}$ is a kernel, there exists an associated RKHS $\H_{K_{\P_0}}$ for which $K_{\P_0}$ is the reproducing kernel; it has the canonical feature map $x\mapsto K_{\P_0}(\cdot,x) \in \H_{K_{\P_0}}$ ($x\in\X$).\footnote{Note that $K_{\P_0}(x,x') = \fip{\sfm(x),\sfm(x')}{\H} = \fip{K_{\P_0}(\cdot,x),K_{\P_0}(\cdot,x')}{\H_{K_{\P_0}}}$; while $\sfm(x)\in\H$ and $K_{\P_0}(\cdot,x)\in\H_{K_{\P_0}}$ ($x\in\X$), both yield the same Stein kernel $K_{\P_0}$.} Denote by $\mathcal T$ the set of probability measures $\P_0\in\M_1^+(\X)$ for which \eqref{eq:sfm=0} can be satisfied and $\H_{K_{\P_0}}$ is separable. The measurability of $\sfm$ is sufficient to guarantee the measurability of $K_{\P_0}$ and $K_{\P_0}(\cdot,x)$  ($x\in \X$) by the separability of $\H_{{K_{\P_0}}}$ \citep[Lemma~4.25]{steinwart08support}; also, $\E_{\P_0}[\sfm(X)] =0$ implies \citep[(15)-(16)]{cribeiro26minimax} that \begin{align}\E_{\P_0}[{K_{\P_0}}(\cdot,X)]=0. \label{eq:Ep_0K_0=0}\end{align}
The kernel Stein discrepancy (KSD; \citealt{hagrass26stein,cribeiro26minimax}) of the target $\P_0$ and the sampling distribution $\P$ is defined as the IPM
\begin{align}
    \KSD(\P_0,\P) &=\sup_{f\in B(\H)}\big|\E_{\P_0}[(T_{\P_0} f)(X)]- \E_{\P}[(T_{\P_0} f)(X)]\big| \\
    &= \norm{\E_{\P}[\sfm(X)]}{\H} = \norm{\E_{\P}[K_{\P_0}(\cdot,X)]}{\H_{K_{\P_0}}},\label{eq:ksd_def}
\end{align}
for $\P_0 \in \mathcal T$ and $\P \in \mP(K_{\P_0},\X)$, where the second equality comes from \citep[(15)-(16)]{cribeiro26minimax}.\footnote{Note that this definition generalizes the well-known Langevin-Stein KSD in $\R^d$ \citep{chwialkowski16kernel,liu16kernelized,oates17control,gorham17measuring}; see \citet[Example~1.1]{hagrass26stein}. KSDs can also be defined using the Pettis integral~\citep{barp24targeted}, which, for simplicity, we do not consider in this paper.}
We say that $K_{\P_0}$ is characteristic to $\mP(K_{\P_0},\X)$ with regard to $\P_0$ if $\KSD(\P_0,\P) = 0$ iff $\P_0 = \P$ holds for all $\P \in \mP(K_{\P_0};\X)$.

\section{Results}\label{sec:results}
This section is dedicated to our results: The minimax lower bound for the estimation of $\MMD$ (Theorem~\ref{th:mmd}), $\HSIC$ (Theorem~\ref{th:hsic}) and $\KSD$ (Theorem~\ref{th:ksd}) on topological spaces under mild assumptions. Further, Theorem~\ref{th:mmd} and Theorem~\ref{th:hsic} allow one to obtain a lower bound for the estimation of the mean embedding (Corollary~\ref{cor:mean_emb}) and the centered cross-covariance operator (Corollary~\ref{cor:cross_cov}), respectively.

Our results fall under the umbrella of minimax estimation \citep{tsybakov09introduction} which we recall in the following. Let $\hat F_n =\hat F_n(X_{1:n})$ be an estimator based on $n$ i.i.d. samples from $\P$ of a functional $F=F(\P)$. A positive sequence $(a_n)_{n=1}^{\infty}$ is said to be a lower bound of $F$ estimation with regard to a class $\mathcal P$ of Borel probability measures on $(\X,\tau_{\X})$ if there exists a constant $B>0$ and threshold $n_0 \in \Np$ such that for all $n\ge n_0$, one has
\begin{align}
    \inf_{\hat F_n}\sup_{\P\in\mathcal P} \P^n\Big(\big|F-\hat F_n\big|\geq a_nB\Big)>0. \label{eq:mlb}
\end{align}
If some specific estimator $\hat F_n^*$ of $F$ has an upper bound in probability that matches $(a_n)_{n=1}^{\infty}$ up to constants---that is, $|F-\hat F_n^
*|=\mathcal O_{\P}(a_n)$---we say that $\hat F_n^*$ is minimax optimal w.r.t.\ $\mathcal P$.

We use Le Cam's two-point method (\citealt{lecam73convergence}; recalled in Theorem~\ref{th:le-cam}) to obtain bounds as in~\eqref{eq:mlb}; estimators for the studied
discrepancies with $\sqrt n$-consistency were recalled in Section~\ref{sec:intro}.
The core idea of this technique is to reduce the problem of finding a lower bound over a large class of distributions $\P \in  \mathcal P$ to the problem of finding a carefully crafted adversarial sequence of distributions; the key technical challenge and contribution of our work is the construction of this adversarial sequence. In order to apply Le Cam's method one must show that there exists $\alpha >0$ and $n_0\in\Np$ such that for all $n\geq n_0$ there is an adversarial pair of distributions $(\P_1,\P_2)=(\P_1(n),\P_2(n))\in \mathcal P^2$ and $s_n >0$ for which
\begin{enumerate}
    \item $\KL(\P_1^n\|\P_2^n)\le\alpha$: their $n$-fold products are similar in KL sense, while \label{en:2.}
   \item $|F(\P_1)-F(\P_2)|\geq 2s_n$: the target functional $F$ for $\P_1$ and $\P_2$ is dissimilar. \label{en:1.}
\end{enumerate}
In this case, $\inf_{\hat F_n}\sup_{\P\in\mathcal P}\P^n\Big(\big|F-\hat F_n\big|\geq s_n\Big)\geq \max\Big(\frac{e^{-\alpha}}{4},\frac{1-\sqrt{\alpha/2}}{2}\Big)$ for all $n\geq n_0$. Hence, to establish the minimax optimality of existing estimators of $\MMD$, $\HSIC$ and $\KSD$, it is sufficient to find an adversarial pair for which \ref{en:2.}.\ holds for some $\alpha>0$ and satisfy \ref{en:1.}.\ with $s_n=\Theta( n^{-1/2})$.
Since we seek to derive these bounds for any topological space $(\X,\tau_\X)$ we can not select any specific distribution (such as Gaussian). Hence, we will employ general perturbation functions $\varphi\in\mC_b(\X)$ with carefully selected properties to construct our adversarial pairs; the assumptions imposed on $(\X,\tau_\X)$ guarantee the existence of such perturbations in each case.

Our proofs rely on Le Cam's method paired with perturbed measures. The latter construction has found various successful applications in  machine learning and statistics \citep{anderson94twosample,sriperumbudur10hilbert,gretton12kernel,balasubramanian21optimality,hagrass26stein}; in particular in the context of hypothesis testing. In this work we consider perturbed measures for the estimation of kernel-based information theoretical measures. Our lower bounds apply to general topological spaces, on which we also establish the existence of the required perturbations.

We now state our main results.

\subsection{Maximum Mean Discrepancy}

Let $(\X,\tau_{\X})$ be a topological space, $K:\X^2 \to \R$ a kernel on $\X$ with associated RKHS $\H_K$. We impose the following assumption.
\begin{assumption}[]\label{as:mmd}
    Assume that (i) $K$ is characteristic to $\mP(K;\X)$ and (ii) there exists a pair $(\P_0, \varphi_0)\in\mP(K;\X) \times \mathcal C_b(\X)$ such that there is no $c\in\R$ for which $\varphi_0=c$ holds $\P_0$-almost surely (a.s.).
\end{assumption}

Assumption~\ref{as:mmd}(i) is commonly used in the literature~\citep{sejdinovic13equivalence,simon-gabriel18kernel};
the next lemma shows that Assumption~\ref{as:mmd}(ii) holds under mild conditions.
\begin{lemma}[Sufficient conditions for Assumption~\ref{as:mmd}(ii)]\label{lem:sufficient_cond}
Let $(\X,\tau_{\X})$ be a metric separable topological space and $K\in\mC_b(\X^2)$. Assume that $\H_K$ has a non-constant element. Then, there exists $(\P_0,\varphi_{0})\in\mP(K;\X)\times\mC_b(\X)$ satisfying Assumption~\ref{as:mmd}(ii).
\end{lemma}

In the following theorem, we establish the minimax lower bound of $\MMD$ estimation.
\begin{theorem}[Minimax lower bound of $\MMD$]\label{th:mmd}
    Let Assumption~\ref{as:mmd} hold. Then, there exists a constant $B>0$ such that \begin{align}
        \liminf_{n,m\rightarrow\infty} \inf_{\hat F_{n,m}} ~ \sup_{(\P,\Q)\in[\mP(K;\X)]^2}\left(\P^n\otimes\Q^m\right)\Big(\big|\MMD(\P,\Q)-\hat F_{n,m}\big|\geq B\big(m^{-1/2} + n^{-1/2}\big) \Big) >0,\label{eq:thm-mmd}
    \end{align}
    where $\hat F_{n,m} \coloneq \hat F_{n,m}(X_{1:n},Y_{1:m})$ denotes any estimator of $\MMD(\P,\Q)$ based on sample $(X_{1:n},Y_{1:m})\sim \P^n \otimes \Q^m$.
\end{theorem}
We comment on this result after stating a direct corollary.

\begin{corollary}[Minimax lower bound of $\mu_K$]\label{cor:mean_emb}
Let Assumption~\ref{as:mmd} hold. Then, there exists a constant $B>0$ such that  \begin{align}
     \liminf_{n\to\infty} \inf_{\hat G_n} \sup_{\P\in\mP(K;\X)} \P^n\left(\norm{\mu_K(\P) - \hat G_n}{\H_K}\geq Bn^{-1/2}\right)>0,
\end{align}
where $\hat G_n\coloneq \hat G_n(X_{1:n})$ denotes any estimator of $\mu_K(\P)$ based on sample $X_{1:n} \sim \P^n$.
\end{corollary}

The following remark elaborates both results.
\begin{remark}
    We recall related minimax lower bounds with their imposed assumptions in the following, emphasizing the considerably broader applicability of our results.
    \begin{itemize}
        \item \tb{Difference to known minimax lower bound of MMD.} \citet{tolstikhin16minimax2} consider $\X=\R^d$ and  continuous radial characteristic kernels. \citet[Theorem~4]{chamakh24keep} worked with the exponential kernel on $\X=\R$.
        \item \tb{Difference to known minimax lower bound of mean embedding}. \citet{tolstikhin17minimax} presented minimax lower bounds (in the $\H_K$-norm) for estimating the mean embedding assuming continuous translation-invariant characteristic kernels on $\X=\R^d$~\citep[Theorem~1 and Theorem~6]{tolstikhin17minimax}. %

    \end{itemize}
\end{remark}

\subsection{Hilbert-Schmidt Independence Criterion}

Fix $d\in \Np$. For each $i\in[d]$, let $(\X_i,\tau_{\X_i})$ be a topological space and $\H_{K_i}$ an RKHS of $\X_i\to \R$ functions with kernel $K_i : \X_i^2\to \R$. Denote the product space by $\X=\times_{i=1}^d\X_i$ and the tensor product RKHS by $\H_K=\otimes_{i=1}^d\H_{K_i}$; recall that $\H_K$ is an RKHS with the product kernel $K=\otimes_{i=1}^dK_i$. We impose the following assumption.

\begin{assumption}\label{as:hics}
        Assume that (i) $K$ is $\mathcal I$-characteristic and  (ii) for each $i\in[d]$, there exist a pair $(\P_i,\varphi_i)\in\mP(K_i;\X_i)\times\mC_b\big(\X_i\big)$  for which there is no $c_i\in\R$ for which $\varphi_i = c_i$ holds $\P_i$-a.s.

\end{assumption}
Note that Assumption~\ref{as:hics}(ii) holds as long as Assumption~\ref{as:mmd}(ii) holds for each $\X_i$ and $K_i$. Further, for $d=2$ the characteristic property of $K_i$-s ($i\in [2]$) is equivalent to the $\mathcal I$-characteristic property of $\otimes_{i=1}^2 K_i$. For $d>2$, $\otimes_{i=1}^d K_i$ is \emph{not} necessarily $\mathcal I$-characteristic for characteristic $K_i$-s ($i\in [d]$; \citealt[Fig.~1, Example~2]{szabo18characteristic2}). However, if $\X_i$ is a locally compact Polish space and $K_i$ is $c_0$-universal for all $i\in[d]$, then $\otimes_{i=1}^d K_i$ is $c_0$-universal \citep[Theorem~5]{szabo18characteristic2} and hence it is $\mathcal I$-characteristic. In case of  $\X_i = \R^{m_i}$ and continuous translation-invariant characteristic kernels $K_i$ ($i\in [d]$),  $\otimes_{i=1}^d K_i$ is $\mathcal I$-characteristic \citep[Theorem~4]{szabo18characteristic2}.

The following theorem settles the minimax lower bound of $\HSIC$ estimation.

\begin{theorem}[Minimax lower bound of $\HSIC$]\label{th:hsic}
    Let Assumption~\ref{as:hics} hold. Then, there exists a constant $B>0$ such that
    \begin{align}
        \liminf_{n\to\infty}\inf_{\hat F_n}\sup_{\P\in\mP(K;\X)}\P^n\Big(\left|\HSIC(\P)-\hat F_n\right|\geq Bn^{-1/2}\Big)>0,
    \end{align}
    where $\hat F_n \coloneq \hat F_n(X_{1:n})$ stands for any estimator of $\HSIC(\P)$ based on sample $X_{1:n}\sim \P^n$.
\end{theorem}

We comment on this result after stating a direct corollary.

\begin{corollary}[Minimax lower bound of $C_K$]\label{cor:cross_cov}
Let Assumption~\ref{as:hics} hold. Then, there exists a constant $B>0$ such that
\begin{align}
    \liminf_{n\to\infty} \inf_{\hat F_n}\sup_{\P\in\mP(K;\X)} \P^n\left(\norm{C_K(\P) -\hat G_n}{\H_K}\geq Bn^{-1/2}\right)>0,
\end{align}
where $\hat G_n\coloneq \hat G_n(X_{1:n})$ stands for any estimator of $C_K(\P)$ based on sample $X_{1:n}\sim \P^n$.
\end{corollary}

We elaborate both results in the following remark.
\begin{remark} As before, we elaborate on the differences compared to related minimax lower bounds in the following; our results relax their assumptions significantly.
    \begin{itemize}
        \item \tb{Difference to known minimax lower bound of HSIC and $C_K$.} \citet[Theorem 1 and Corollary 1]{kalinke24minimax} assumed that $\X = \R^d$ and continuous translation-invariant characteristic kernels.

        \item \tb{Difference to known minimax lower bound of the covariance operator.} Corollary~\ref{cor:cross_cov} also provides lower bounds for the related problem of estimating the centered covariance operator (by noting that the centered covariance is the centered cross-covariance of a random variable with itself). Hence, our results also relax the assumptions of $\X = \R^d$ and continuous translation-invariant characteristic kernels imposed by \citet{zhou19class} and  \citet[Corollary 1]{kalinke24minimax}.
    \end{itemize}
\end{remark}

\subsection{Kernel Stein Discrepancy}
Let $(\X,\tau_\X)$ be a topological space, $\P_0\in\M_1^+(\X)$, $K_{\P_0}:\X^2\to\R$ a Stein kernel%
, and $\H_{K_{\P_0}}$ the RKHS associated with $K_{\P_0}$. We impose the following assumption.

\begin{assumption}{(\citealt[Assumption~4]{cribeiro26minimax})} \label{as:ksd}
Assume that there exists a pair $(\P_0,\varphi_0) \in \mathcal T \times \mC_b(\X)$ such that (i) $K_{\P_0}$ is characteristic to $\mP(K_{\P_0};\X)$ with regard to $\P_0$ and (ii) there is no $c\in\R$ for which $\varphi_0=c$ holds $\P_0$-a.s.
\end{assumption}

Sufficient conditions for Assumption~\ref{as:ksd}(i) are known in $\R^d$ (\citealt[Proposition~3.3]{liu16kernelized}; \citealt[Theorem~2.2]{chwialkowski16kernel}). Further, Assumption~\ref{as:ksd}(ii) is the same as Assumption~\ref{as:mmd}(ii); hence, Lemma~\ref{lem:sufficient_cond} also states sufficient conditions for it.

\begin{theorem}\label{th:ksd}\emph{(Minimax lower bound of $\KSD$; \citealt[Theorem~2]{cribeiro26minimax})}
    Let Assumption~\ref{as:ksd} hold. Then, there exists a constant $B>0$ such that \begin{align}
       \liminf_{n\to\infty} \inf_{\hat F_n}\sup_{\P_0\in\mathcal T}\sup_{\P\in\mP(K_{\P_0};\X)} \P^n\Big(\left|\KSD(\P_0,\P)-\hat F_n\right|\geq Bn^{-1/2}\Big)>0, \label{eq:ksd_lb}
    \end{align}
    where $\hat F_n \coloneq \hat F_n(X_{1:n})$ denotes any estimator $\KSD(\P_0,\P)$ using sample $X_{1:n}\sim\P^n$.
\end{theorem}
\begin{remark}~

\begin{itemize}
    \item \tb{Double supremum.} Note that $K_{\P_0}$ depends on the chosen $\P_0\in\mathcal T$ and hence, so does $\mP(K_{\P_0};\X)$. This is the reason of the double supremum in~\eqref{eq:ksd_lb}.
    \item \tb{Difference with known minimax lower bound of KSD.} \citet{gogolashvili26minimax} assumed $\X = \R^d$ and for measures to be in $S(K^2;\X)=\big\{\P\!\in\!
    \M_1^+(\X)\!:~\int_{\X^2}\!K^2(x,y)\d(\P\otimes\P)(x,y)<\infty\big\}$.
    \citet[Theorem~1]{cribeiro26minimax} also assumed $\X = \R^d$ and Stein kernels with characteristic translation-invariant base kernels.

\end{itemize}
\end{remark}

\section*{Acknowledgments}
FK is supported by the KiKIT (The Pilot Program for Core-Informatics at the KIT) of the Helmholtz Association.
\newpage

\appendix

\section{Proofs}\label{sec:proofs}

\begin{figure}[!ht]
\centering
\resizebox{1\textwidth}{!}{%
\begin{circuitikz}
\tikzstyle{every node}=[font=\fontsize{18.2pt}{23.7pt}\selectfont]
\node [font=\fontsize{18.2pt}{23.7pt}\selectfont, inner xsep=0.080cm, inner ysep=0.085cm, rounded corners=0.020cm] at (13,7) {Theorem~\ref{th:mmd}};
\node [font=\fontsize{18.2pt}{23.7pt}\selectfont, inner xsep=0.080cm, inner ysep=0.085cm, rounded corners=0.020cm] at (20,7) {Theorem~\ref{th:ksd}};
\node [font=\fontsize{18.2pt}{23.7pt}\selectfont, inner xsep=0.080cm, inner ysep=0.085cm, rounded corners=0.020cm] at (26.875,7) {Theorem~\ref{th:hsic}};
\node [font=\fontsize{18.2pt}{23.7pt}\selectfont, inner xsep=0.080cm, inner ysep=0.085cm, rounded corners=0.020cm] at (12.875,10.75) {\ref{lem:p_n-and-q_m-inP}};
\node [font=\fontsize{18.2pt}{23.7pt}\selectfont, inner xsep=0.080cm, inner ysep=0.085cm, rounded corners=0.020cm] at (17.5,10.75) {\ref{lem:kl-bound-perturb}};
\node [font=\fontsize{18.2pt}{23.7pt}\selectfont, inner xsep=0.080cm, inner ysep=0.085cm, rounded corners=0.020cm] at (22.5,10.75) {\ref{th:le-cam}};
\node [font=\fontsize{18.2pt}{23.7pt}\selectfont, inner xsep=0.080cm, inner ysep=0.085cm, rounded corners=0.020cm] at (27,10.75) {\ref{lem:tsy-kl-prod-measure}};
\node [font=\fontsize{18.2pt}{23.7pt}\selectfont, inner xsep=0.080cm, inner ysep=0.085cm, rounded corners=0.020cm] at (12.875,12.75) {\ref{lem:p_n-and-q_m-inM}};
\draw [-{Stealth[scale=1.5]}, ] (12.875,12.375) -- (12.875,11.25);
\draw [-{Stealth[scale=1.5]}, ] (20,10) -- (13.25,7.5);
\draw [-{Stealth[scale=1.5]}, ] (20,10) -- (20,7.5);
\draw [-{Stealth[scale=1.5]}, ] (20,10) -- (26.875,7.625);
\draw [short] (17.5,10.5) -- (17.5,10);
\draw [short] (22.5,10.5) -- (22.5,10);
\draw [short] (13,10) -- (27,10);
\draw [short] (13,10) -- (13,10.5);
\draw [short] (27,10) -- (27,10.5);
\node [font=\fontsize{18.2pt}{23.7pt}\selectfont, fill={rgb,255:red,255; green,255; blue,255}, fill opacity=1, text opacity=1, inner xsep=0.080cm, inner ysep=0.085cm, rounded corners=0.020cm] at (20,14.5) {\ref{lem:mu_k-perturbed}};
\node [font=\fontsize{18.2pt}{23.7pt}\selectfont, fill={rgb,255:red,255; green,255; blue,255}, fill opacity=1, text opacity=1, inner xsep=0.080cm, inner ysep=0.085cm, rounded corners=0.020cm] at (20,3.25) {\ref{lem:phi-exists}};
\node [font=\fontsize{18.2pt}{23.7pt}\selectfont, fill={rgb,255:red,255; green,255; blue,255}, fill opacity=1, text opacity=1, inner xsep=0.080cm, inner ysep=0.085cm, rounded corners=0.020cm] at (26.125,5.125) {\ref{lem:phi-exists-X1xX2}};
\node [font=\fontsize{18.2pt}{23.7pt}\selectfont, fill={rgb,255:red,255; green,255; blue,255}, fill opacity=1, text opacity=1, inner xsep=0.080cm, inner ysep=0.085cm, rounded corners=0.020cm] at (28,5.125) {\ref{lem:marginals_pn}};
\node [font=\fontsize{18.2pt}{23.7pt}\selectfont, fill={rgb,255:red,255; green,255; blue,255}, fill opacity=1, text opacity=1, inner xsep=0.080cm, inner ysep=0.085cm, rounded corners=0.020cm] at (23.75,5.125) {\ref{lem:p_n-neq-p_0}};
\node [font=\fontsize{18.2pt}{23.7pt}\selectfont, fill={rgb,255:red,255; green,255; blue,255}, fill opacity=1, text opacity=1, inner xsep=0.080cm, inner ysep=0.085cm, rounded corners=0.020cm] at (26.125,3.25) {\ref{lem:open-A}};
\node [font=\fontsize{18.2pt}{23.7pt}\selectfont, fill={rgb,255:red,255; green,255; blue,255}, fill opacity=1, text opacity=1, inner xsep=0.080cm, inner ysep=0.085cm, rounded corners=0.020cm] at (28,3.25) {\ref{th::ex::Fubini}};
\draw [-{Stealth[scale=1.5]}, ] (26.25,3.875) -- (26.25,4.75);
\draw [-{Stealth[scale=1.5]}, ] (28,3.875) -- (28,4.75);
\draw [-{Stealth[scale=1.5]}, ] (26.125,5.75) -- (26.125,6.625);
\draw [-{Stealth[scale=1.5]}, ] (27.875,5.75) -- (27.875,6.625);
\draw [-{Stealth[scale=1.5]}, ] (23.5,5.625) -- (21.625,6.625);
\draw [-{Stealth[scale=1.5]}, ] (24,5.625) -- (25.5,6.625);
\draw [-{Stealth[scale=1.5]}, ] (20.625,3.375) -- (25.375,5);
\draw [-{Stealth[scale=1.5]}, ] (20,3.875) -- (20,6.5);
\draw [-{Stealth[scale=1.5]}, ] (19.375,3.5) -- (14.625,6.5);
\node [font=\fontsize{18.2pt}{23.7pt}\selectfont, fill={rgb,255:red,255; green,255; blue,255}, fill opacity=1, text opacity=1, inner xsep=0.080cm, inner ysep=0.085cm, rounded corners=0.020cm] at (12.875,3.25) {\ref{cor:p_n-neq-q_m}};
\draw [-{Stealth[scale=1.5]}, ] (13,3.875) -- (13,6.5);
\node [font=\fontsize{18.2pt}{23.7pt}\selectfont, fill={rgb,255:red,255; green,255; blue,255}, fill opacity=1, text opacity=1, inner xsep=0.080cm, inner ysep=0.085cm, rounded corners=0.020cm] at (7.5,7) {Corollary~\ref{cor:mean_emb}};
\node [font=\fontsize{18.2pt}{23.7pt}\selectfont, fill={rgb,255:red,255; green,255; blue,255}, fill opacity=1, text opacity=1, inner xsep=0.080cm, inner ysep=0.085cm, rounded corners=0.020cm] at (32.5,7) {Corollary~\ref{cor:cross_cov}};
\node [font=\fontsize{18.2pt}{23.7pt}\selectfont, fill={rgb,255:red,255; green,255; blue,255}, fill opacity=1, text opacity=1, inner xsep=0.080cm, inner ysep=0.085cm, rounded corners=0.020cm] at (17.375,0.75) {Lemma~\ref{lem:sufficient_cond}};
\node [font=\fontsize{18.2pt}{23.7pt}\selectfont, fill={rgb,255:red,255; green,255; blue,255}, fill opacity=1, text opacity=1, inner xsep=0.080cm, inner ysep=0.085cm, rounded corners=0.020cm] at (21.875,0.75) {\ref{lem:phi-not-a.s.c.}};
\draw [-{Stealth[scale=1.5]}, ] (11.25,7) -- (9.125,7);
\draw [-{Stealth[scale=1.5]}, ] (28.75,7) -- (30.875,7);
\draw [-{Stealth[scale=1.5]}, ] (21.25,0.75) -- (18.75,0.75);
\draw [short] (19.5,14.625) -- (11.875,14.625);
\draw [short] (20.5,14.625) -- (28.125,14.625);
\draw [-{Stealth[scale=1.5]}, ] (11.875,14.625) -- (11.875,7.5);
\draw [-{Stealth[scale=1.5]}, ] (28.125,14.625) -- (28.125,7.5);
\end{circuitikz}
}%
\caption{Summary of the dependencies of our results. $R_1 \leftarrow R_2$ means that ``result $R_1$ depends on $R_2$''. \ref{lem:p_n-and-q_m-inP}, \ref{lem:kl-bound-perturb}, \ref{th:le-cam} and \ref{lem:tsy-kl-prod-measure} are all used in Theorems~\ref{th:mmd}, \ref{th:hsic} and \ref{th:ksd}.}
\label{fig:my_label}
\end{figure}
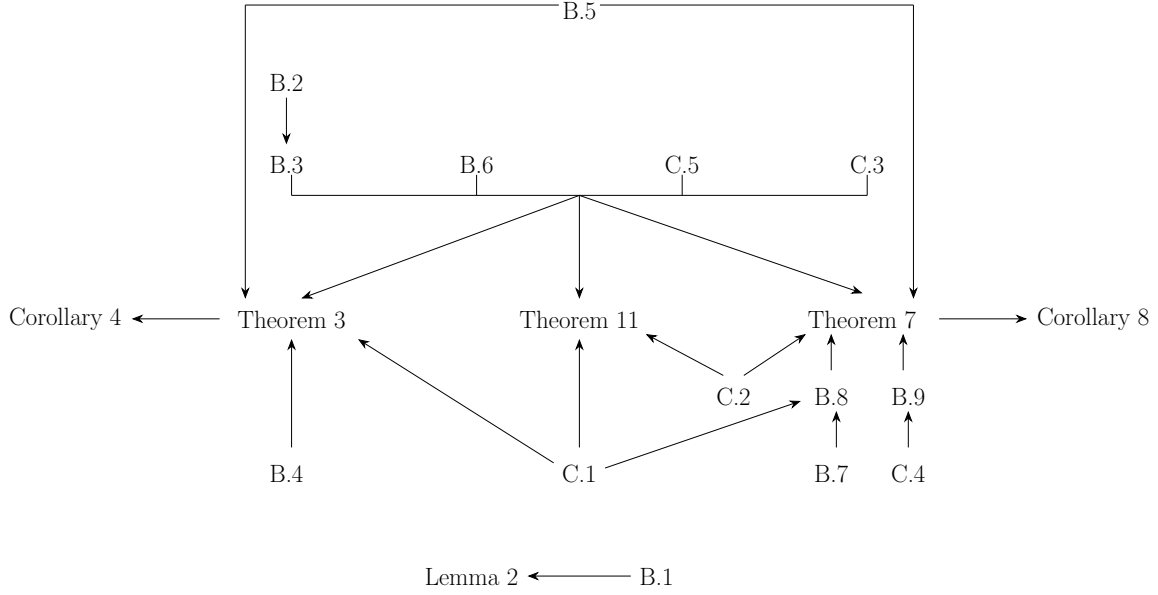

\subsection{Proof of Lemma~\ref{lem:sufficient_cond}}
Let $\varphi_0\in\H_K\subseteq\mC_b(\X)$ be the assumed non-constant element, where the inclusion $\H_K\subseteq\mC_b(\X)$ holds by $K\in\mC_b(\X^2)$ \citep[Lemma~4.28]{steinwart08support}.
Further, as $\X$ is separable, there exists a countable dense subset $\{x_i\}_{i=1}^{\infty}\subseteq\X$. Let $\P_0=\sum_{i=1}^\infty2^{-i}\delta_{x_i}\in\M_1^+(\X)$.
Notice that $\supp(\P_0)$ is well-defined as $\X$ is separable and metric, which implies second-countability.$^\text{\ref{foot:2ndcount}}$ Further, $\{x_i\}\subseteq \supp(\P_0)$ for all $i\in\Np$. We argue by contradiction: assume that $\{x_i\}\subseteq \X\setminus\supp(\P_0)$. This implies that $2^{-i}\jover a =\P_0(\{x_i\})\jover b\leq \P_0(\X\setminus\supp(\P_0)) = 0$, where (a) comes by the definition of $\P_0$, and (b)  holds by the monotonicity of probability measures; comparing the l.h.s.\ and the r.h.s.\ of the inequality leads to a contradiction.
Then,
\begin{align}
\overline{\cup_{i=1}^\infty \{x_i\}} \jover{a}{\subseteq} \supp(\P_0) \jover{b}{\subseteq} \X \jover{c}{=} \overline{\cup_{i=1}^\infty \{x_i\}} \implies \supp(\P_0) = \X.
\end{align}
where (a) holds as $\cup_{i=1}^\infty\{x_i\}\subseteq \supp(\P_0)$ (as shown above) and as $\supp(\P_0)$ is closed (by definition), (b) comes from the definition of the support, (c) is by the denseness of $\{x_i\}_{i=1}^{\infty}$ in $\X$.
The boundedness of $K$ following from $K\in\mC_b(\X^2)$ implies that $\M_1^+(\X)=\mP(K;\X)$ and hence that $\P_0\in\mP(K;\X)$. Moreover, $\P_0$ having full-support guarantees that $\varphi_0$ is not $\P_0$-a.s.\ constant by Lemma~\ref{lem:phi-not-a.s.c.}. This $(\P_0,\varphi_0)$-pair proves the claim.

\subsection{Proof of Theorem~\ref{th:mmd}}\label{sec:proof_th1}

This proof relies on Le Cam's two-point method (recalled in Section~\ref{sec:results}); following Le Cam's approach, the proof has 3 parts: after defining the adversarial pair $(\theta_1,\theta_2)$, we then control the distance $d(\theta_1,\theta_2)=|\theta_1-\theta_2|$, followed by bounding the divergence $\KL(\P_{\theta_2}\|\P_{\theta_1})$.

The \tb{definition of $(\theta_1,\theta_2)$} relies on certain auxiliary quantities $\P_0$, $\varphi_0$, $\P_{(n)}$ and $\Q_{(m)}$ which we define in the following.
Let $(\P_0,\varphi_0)\in\mP(K;\X)\times \mC_b(\X)$ satisfy Assumption~\ref{as:mmd}(ii) and let $\varphi\in \mC_b(\X)$ be a perturbation function associated with $(\P_0,\varphi_0)$ according to Lemma~\ref{lem:phi-exists} [recall that $\mP(K;\X)\subseteq\M_1^+(\X)$], satisfying
$\E_{\P_0}[\varphi(X)] = 0$ and such that $\varphi(x)\neq0$ on all points of a set of positive $\P_0$-measure.
For $A\in \mathcal B(\X)$, let \begin{align}
        \P_{(n)}(A) = \int_A 1+\varepsilon_n \varphi(x) \d\P_0(x),\quad  \Q_{(m)}(A) = \int_A 1-\nu_m \varphi(x) \d\P_0(x), \label{eq:p_n_and_q_m_def}
    \end{align}
    where $\varepsilon_n = c_1n^{-1/2}$, $\nu_m = c_2m^{-1/2}$, and $c_1$ and $c_2$ are arbitrary positive constants. Note that there exist $n_{0,1},m_{0,1}\in \Np$ such that for all $n\ge n_{0,1}$ and $m\ge m_{0,1}$ it holds that $\P_{(n)},\Q_{(m)}\in\mP(K;\X)$  (by Lemma~\ref{lem:p_n-and-q_m-inP}) and $\P_{(n)}\neq\Q_{(m)}$ (by Lemma~\ref{cor:p_n-neq-q_m}).

With these notations at hand, to bring ourselves into the setting of Theorem~\ref{th:le-cam}, consider arbitrary $n\ge n_{0,1}$ and $m\ge m_{0,1}$, define the sample space $\Y = \X^n\times\X^m$, $\Theta \coloneq  \big\{\theta_{\P,\Q} \coloneq \MMD(\P,\Q):~(\P,\Q)\in [\mP(K;\X)]^2\}$ and $\mathcal P_\Theta \coloneq \big\{ \P^n \otimes\Q^m:~(\P,\Q)\in[\mP(K;\X)]^2\big\} = \big\{\P^n\otimes\Q^m:~\theta_{\P,\Q}\in\Theta \big\}$. Let $F:[\mP(K;\X)]^2 \to \R$ be defined as $(\P,\Q)\mapsto \MMD(\P,\Q)$ and let $\hat F_{n,m}$ be a corresponding estimator with $n$ and $m$ samples from $\P$ and $\Q$, respectively. Define the metric $d(\theta_1,\theta_2)\coloneq |\theta_1-\theta_2|$ and the adversarial pair of hypotheses $(\theta_1,\theta_2)=\big(\theta_{1}(n,m),\theta_{2}(n,m)\big) \coloneq (\theta_{\P_0,\P_0},\theta_{\P_{(n)},\Q_{(m)}})$. By the definition of $\mathcal P_\Theta$ we have  $\left(\P_{\theta_{1}},\P_{\theta_{2}}\right)=\big(\P_{\theta_{\P_0,\P_0}},\P_{\theta_{\P_{(n)},\Q_{(m)}}}\big) = \big(\P_0^n\otimes\P_0^m,\P_{(n)}^n\otimes\Q_{(m)}^m\big)$.\footnote{We note that $\P_{(n)}^n\otimes\Q_{(m)}^m$ is well-defined as $\P_{(n)},\Q_{(m)}\in \M_1^+(\X)$.}

We now proceed by \tb{controlling the distance $d(\theta_1,\theta_2)$}:
    \begin{align}
       d\big(\theta_{\P_0,\P_0},\theta_{\P_{(n)},\Q_{(m)}}\big)
       \jover{a}{=}& \MMD(\P_{(n)},\Q_{(m)}) = \norm{\mu_K(\P_{(n)}) - \mu_K(\Q_{(m)})}{\H_K} \\
         \jover{b}{=} &\norm{\mu_K(\P_0)\! + \!\varepsilon_n\E_{\P_0}\big[K(\cdot,X)\varphi(X)\big]\! - \!\mu_K(\P_0)\!+\!\nu_m\E_{\P_0}\big[K(\cdot,X)\varphi(X)\big]}{\H_K} \\
        \jover{c}{=} & \norm{\varepsilon_n\E_{\P_0}\big[K(\cdot,X)\varphi(X)\big]  + \nu_m\E_{\P_0}\big[K(\cdot,X)\varphi(X)\big]}{\H_K} \\
        \jover{d}{=}& (\varepsilon_n+\nu_m)\underbrace{\norm{\E_{\P_0}\big[K(\cdot,X)\varphi(X)\big]}{\H_K}}_{\eqcolon C_{\varphi}} = (\varepsilon_n+\nu_m) C_\varphi \jover{e}> 0, \label{eq:mmd_bound}
    \end{align}
    where (a) follows from $\MMD(\P_0,\P_0)=0$ and the non-negativity of $\MMD$ (holding by its definition), (b) comes by applying Lemma~\ref{lem:mu_k-perturbed} to both terms inside the norm, in (c) the expression within the norm was simplified, and (d) follows from the absolute homogeneity of norms and the fact that $\varepsilon_n+\nu_m>0$. The injectivity of  $\mu_K$ [holding by Assumption~\ref{as:mmd}(i)] and recalling that $\P_{(n)}\neq\Q_{(m)}$ yield (e) as all steps coming before were equalities.  Hence,
    \begin{align}
    \MMD(\P_{(n)},\Q_{(m)}) &= (\varepsilon_n + \nu_m)C_\varphi \jover{a}{=}\left(c_1n^{-1/2} + c_2m^{-1/2}\right)C_\varphi \jover{b}{\geq} c\left(n^{-1/2} + m^{-1/2}\right)C_\varphi\\
    &\jover{c}{=}  2B\left(n^{-1/2} + m^{-1/2}\right), \label{eq:mmd-dist-cont}
    \end{align}
where (a) comes from the definition of $\varepsilon_n$ and $\nu_m$. (b) and (c) follow from introducing the notations $c\coloneq\min(c_1,c_2)>0$ and $B\coloneq cC_{\varphi}/2>0$.

We continue by \tb{bounding the divergence $\mathrm{KL}(\P_{\theta_2}\|\P_{\theta_1})$}. Indeed,
    \begin{align}
        \hspace{-0.3cm}\KL(\P^n_n\otimes\Q_{(m)}^m\|\P_0^n\otimes\P_0^m)
        &\jover{a}=n\KL(\P_{(n)}\|\P_0) + m\KL(\Q_{(m)}\|\P_0)\jover{b}\leq c^2_1\alpha_1 + c^2_2\alpha_2\eqcolon \alpha,\qquad  \label{eq:kl-bound-proof}
    \end{align}
    where (a) comes from Lemma~\ref{lem:tsy-kl-prod-measure}, and (b) follows from the facts that $\KL(\P_{(n)}\|\P_0)\leq c_1^2\ \alpha_1/n$ for $n\geq n_{0,2}\in\Np$ and $\KL(\Q_{(m)}\|\P_0)\leq c_2^2\ \alpha_2/m$ for $m\geq m_{0,2}\in\Np$, with some $\alpha_1,\alpha_2>0$ (by Lemma~\ref{lem:kl-bound-perturb}) giving rise to an $\alpha>0$ by $c_1,c_2,\alpha_1,\alpha_2>0$.

    Then, taking $n_0= \max\{n_{0,1},n_{0,2}\}$, $m_0 = \max\{m_{0,1},m_{0,2}\}$, and invoking Theorem~\ref{th:le-cam} for all $n\geq n_0$ and $m\geq m_0$ [possible by our results \eqref{eq:mmd-dist-cont} and \eqref{eq:kl-bound-proof}] proves that
    \begin{equation}
        \inf_{\hat F_{n,m}} ~ \sup_{(\P,\Q)\in [\mP(K;\X)]^2} \!(\P^n\otimes\Q^m)\!\Big(\big|\MMD(\P,\Q)-\hat F_{n,m}\big|\!\geq\! B\big(n^{-1/2} + m^{-1/2})\Big) \ge f(\alpha),
    \end{equation}
    with $f(\alpha) =  \max \!\big\{\exp(-\alpha) / 4,(1-\sqrt{\alpha/2})\big\}>0$.

\subsection{Proof of Corollary~\ref{cor:mean_emb}}
We follow the same overall strategy as in the proof of Theorem~\ref{th:mmd} (Section~\ref{sec:proof_th1}), but, as estimating $\mu_K$ only requires a sample of a single distribution, we consider a setup that does not involve $m \in \Np$. Indeed, let $n\in\Np$, the sample space $\Y\coloneq \X^n$,  $\Theta \coloneq \big\{\theta_{\P}\coloneq \mu_K(\P):~ \P\in\mP(K;\X)\big\}$, and $\mathcal P_\Theta \coloneq \big\{ \P^n:~\P\in\mP(K;\X)\big\}=\big\{\P^n:~\theta_\P\in\Theta\big\}$. Further, let $F:\mP(K;\X)\to\H_K$ be defined as $\P\mapsto \mu_K(\P)$, and let $\hat F_n$ be a corresponding estimator with $n$ samples. For a pair of hypotheses $(\theta_1,\theta_2)\in\Theta^2$, we define the distance $d(\theta_1,\theta_2)\coloneq \norm{\theta_1-\theta_2}{\H_K}$.

We first \tb{define the adversarial pair} $(\theta_1,\theta_2)\in\Theta^2$.
Consider $\P_0$ and $\P_{(n)}$ as in~\eqref{eq:p_n_and_q_m_def}, and define $(\theta_1,\theta_2)\coloneq(\theta_1(n),\theta_2(n))\coloneq (\theta_{\P_0},\theta_{\P_{(n)}})$. By Lemma~\ref{lem:p_n-and-q_m-inP}, there exists $n_{0,1}\in\Np$ such that $\P_{(n)}\in\mP(K;\X)$ for all $n\geq n_{0,1}$;
hence, by the definition of $\mathcal P_\Theta$, we then have that $\big(\P_{\theta_1},\P_{\theta_2}\big)=\big(\P_{\theta_{\P_0}},\P_{\theta_{\P_{(n)}}}\big)=\big(\P_0^n,\P_{(n)}^n\big)$.

We now proceed by \tb{controlling the distance $d(\theta_1,\theta_2)$:}
\begin{align}
    d(\theta_{\P_0},\theta_{\P_{(n)}}) &= \norm{\mu_K(\P_0)-\mu_K(\P_{(n)})}{\H_K} \jover a {=} \norm{\mu_K(\P_0)-\mu_K(\P_0)-\varepsilon_n\E_{\P_0}[K(\cdot,X)\varphi(X)]}{\H_K}
    \\
    &\jover b =  \varepsilon_n\underbrace{\norm{\E_{\P_0}[K(\cdot,X)\varphi(X)] }{\H_K}}_{\eqcolon C_\varphi} \jover c =
    \underbrace{c_1 C_\varphi}_{\eqcolon  2B} n^{-1/2}   = 2Bn^{-1/2} \overset{\eqref{eq:mmd_bound}}{>} 0,
\end{align}
where (a) is by Lemma~\ref{lem:mu_k-perturbed}, (b) follows by simplifying and the absolute homogeneity of norms combined with $\varepsilon_n>0$, and (c) holds by the definition of $\varepsilon_n$.

Note that $\KL(\P_{(n)}^n\|\P_0^n) \jover a = n\KL(\P_{(n)}\|\P_0) \jover b \leq n\frac{c_1^2}{n}\alpha= c_1^{2}\alpha <\infty$ \tb{bounds the divergence} for all $n\geq n_{0,2}$, where (a) comes from Lemma~\ref{lem:tsy-kl-prod-measure}, and (b) holds by Lemma~\ref{lem:kl-bound-perturb}. Finally, taking $n_0 = \max\{n_{0,1},n_{0,2}\}$, and invoking Theorem~\ref{th:le-cam} for all $n\geq n_0$, proves the statement.

\subsection{Proof of Theorem~\ref{th:hsic}}\label{sec:proof_th2}

As elaborated in the proof of Theorem~\ref{th:mmd} (Section~\ref{sec:proof_th1}), we bring ourselves into the setting of Theorem~\ref{th:le-cam}.
Let the sample space $\Y \coloneq \X^n$, the parameter space $\Theta \coloneq \{\theta_\P \coloneq \HSIC(\P):\ \P\in\mP(K;\X)\}$, and the indexed family $\mathcal P_\Theta\coloneq \{\P^n:\ \P\in \mP(K;\X)\} = \{\P^n:\ \theta_\P\in\Theta\}$. Further, let us define $F:\mP(K;\X)\to\R$ as $\P\mapsto\HSIC(\P)$, and let $\hat F_n$ be a corresponding estimator with $n$ samples. For a pair of hypotheses $(\theta_1,\theta_2)\in\Theta^2$, we define the distance $d(\theta_1,\theta_2) \coloneq |\theta_1-\theta_2|$. To apply Theorem~\ref{th:le-cam}, we are aiming to find an adversarial pair of hypotheses $(\theta_{1},\theta_{2})\in\Theta^2$ for which $d(\theta_{1},\theta_{2})\geq 2Bn^{-1/2}$ for some $B>0$ while $\KL(\P_{\theta_{2}}\|\P_{\theta_{1}})\leq\alpha$ for some $0 < \alpha < \infty$.

In order to \tb{define the pair $(\theta_1,\theta_2)$}, consider the pairs $(\P_i, \varphi_i)\in\mP(K_i;\X_i)\times\mC_b(\X_i)$ ($i\in[d]$) as given in Assumption~\ref{as:hics}(ii). Further, denote their product as $\P_0=\otimes_{i=1}^d\P_i$ and consider the function $\varphi\in\mC_b(\X)$ given by Lemma~\ref{lem:phi-exists-X1xX2} [valid by Assumption~\ref{as:hics}(ii)]. We will consider the adversarial pair $(\theta_1,\theta_2)=(\theta_{1}(n),\theta_{2}(n))\coloneq(\theta_{\P_0},\theta_{\P_{(n)}})$, where $\P_{(n)}$ is defined as
\begin{align}
    \P_{(n)}(A) = \int_{A} 1+\varepsilon_n\varphi(x) \d\P_0(x),
\end{align}
for $A\in\mathcal B(\X)$, $\varepsilon_n = cn^{-1/2}$, and $c>0$.
Further, by Lemma~\ref{lem:p_n-and-q_m-inP}, there is an $n_{0,1}\in\Np$ such that $\P_{(n)}\in\mP(K;\X)$ for all $n\geq n_{0,1}$. By the definition of $\mathcal P_\Theta$ we have $(\P_{\theta_1},\P_{\theta_2})= (\P_{\theta_{\P_0}},\P_{\theta_{\P_{(n)}}}) =(\P_0^n,\P_{(n)}^n)$.\footnote{We note that the $n$-fold product $\P_{(n)}^n$ is well-defined since $\P_{(n)}\in\mP(K;\X)\subseteq\M_1^+(\X)$.}

We now proceed by \tb{controlling the distance $d(\theta_1,\theta_2)$}.  Indeed,
\begin{align}
    d(\theta_{1},\theta_{2}) & =\Big|\underbrace{\HSIC(\P_0)}_{=0} - \HSIC(\P_{(n)})\Big| \jover{a}{=} \HSIC(\P_{(n)})
    \jover{b}{=}\MMD\big(\P_{(n)},\otimes_{i=1}^d\marg{\P_{(n)}}{i}\big)\\
    &
    \jover{c}{=}\MMD\big(\P_{(n)},\otimes_{i=1}^{d}\P_i\big)
    \jover{d}{=}\norm{\mu_K({\P_{(n)}})-\mu_K({\P_0})}{\H_{K}}\\
    &\jover{e}{=}\norm{\mu_K({\P_0})+\varepsilon_n\E_{\P_0}[K(\cdot,X)\varphi(X)] - \mu_K({\P_0})}{\H_{K}}\\
    &\jover{f}{=}\varepsilon_n\underbrace{\norm{\E_{\P_0}[K(\cdot,X)\varphi(X)]}{\H_{K}}}_{\eqcolon C_\varphi}\jover{g}{>} 0, \label{eq:bound_hsic}
\end{align}
where (a) follows from the fact that  $\HSIC(\P_0)=0$ (since $\P_0 =\otimes_{i=1}^d \P_i$) and the non-negativity of $\HSIC$, (b) comes from the definition of $\HSIC$ in \eqref{eq:hsic-definition}, (c) is by Lemma~\ref{lem:marginals_pn}, and (d) holds by the definition of $\MMD$ together with the fact that $\P_0 = \otimes_{i=1}^d\P_i$ by definition. Lemma~\ref{lem:mu_k-perturbed} yields (e), while simplification and the absolute homogeneity of norms with $\varepsilon_n>0$ yield (f). Lastly, Assumption~\ref{as:hics}(i), $\varepsilon_n>0$, and the fact that $\P_{(n)}\neq \P_0$ (by Lemma~\ref{lem:p_n-neq-p_0}) guarantee that $\HSIC(\P_{(n)})>0$; all steps being equalities this shows (g) and implies also that $C_{\varphi}>0$. Hence,
\begin{align}
    \HSIC(\P_{(n)}) = \varepsilon_n C_\varphi \jover a = cn^{-1/2}C_\varphi \jover b = 2Bn^{-1/2}, \label{eq:HSIC>2s}
\end{align}
where (a) comes from the definition of $\varepsilon_n$ and (b) follows by introducing the notation $B\coloneq c C_\varphi/2>0$.

We now proceed by \tb{bounding the divergence $\KL(\P_{\theta_{2}}\|\P_{\theta_{1}})$.} Recall that $\varphi \in \mC_b(\X)$, $\E_{\P_0}[\varphi(X)] = 0$, and $\varphi(x) \neq 0$ for all $x$ in some set $A\in \mathcal B(\X)$ with positive $\P_0$-measure. Then,
\begin{align}
    \KL(\P_{(n)}^n\|\P^n_0) \jover{a}{=}n\KL(\P_{(n)}\|\P_0) \jover {b}{\leq} n\varepsilon_n^2 \alpha_0\jover {c}{=} n\left(cn^{-1/2}\right)^2 \alpha_0 \jover {d}{\eqcolon} \alpha < \infty,\label{eq:KLboundHSIC}
\end{align}
where Lemma~\ref{lem:tsy-kl-prod-measure} gives (a). Notice that by Lemma~\ref{lem:kl-bound-perturb} there exist $\alpha_0>0$ and $n_{0,2}\in\Np$ such that, for all $n\geq n_{0,2}$, (b) holds.
(c) follows from the definition of $\varepsilon_n$, (d) comes by simplification and by introducing $\alpha \coloneq c^2\alpha_0 >0$ where the positivity follows as $c>0$ and $\alpha_0>0$.

Then, taking $n_0=\max(n_{0,1},n_{0,2})$, we invoke Theorem~\ref{th:le-cam}---validated in \eqref{eq:HSIC>2s} and \eqref{eq:KLboundHSIC}---to obtain that, for all $n\geq n_0$,
\begin{align}
    \inf_{\hat F_n} \sup_{\P\in\mP(K;\X)}\P^n\Big( \Big|\HSIC(\P)-\hat F_n\Big|\geq Bn^{-1/2} \Big) \ge f(\alpha)>0,
\end{align}
with $f(\alpha)=\max\{\exp(-\alpha)/4,(1-\sqrt{\alpha/2})\}$. This proves the statement.

\subsection{Proof of Corollary~\ref{cor:cross_cov}}
We adjust the proof strategy of the proof of Theorem~\ref{th:hsic} (Section~\ref{sec:proof_th2}) to the setting of the corollary.
Specifically, we now consider the parameter space $\Theta \coloneq \{\theta_\P \coloneq C_K(\P):\ \P\in\mP(K;\X)\}$ [with $C_K$ defined in \eqref{eq:cov-op-definition}], define $F:\mP(K;\X)\to\R$ as $\P \mapsto C_K(\P)$, $\hat F_n$ a corresponding estimator using $n$ samples, and use the distance $d(\theta_1,\theta_2) \coloneq \norm{\theta_1-\theta_2}{\H_K}$ ($\theta_1,\theta_2\in\Theta$). The remaining part of the setup is as in Section~\ref{sec:proof_th2}. As $\KL(\P_{(n)}^n\|\P^n_0) \le \alpha < \infty$ with $\alpha> 0$ was proved in \eqref{eq:KLboundHSIC}, it remains to establish a lower bound on $d(\theta_1(n),\theta_2(n))=d(\theta_{\P_0},\theta_{\P_{(n)}})=\norm{C_K(\P_0)-C_K(\P_{(n)})}{\H_K}$. Indeed, \begin{align}
    \norm{C_K(\P_0)-C_K(\P_{(n)})}{\H_K}&\jover a\geq \big|\underbrace{\norm{C_K(\P_0)}{\H_K}}_{= \HSIC(\P_0)}-\underbrace{\norm{C_K(\P_{(n)})}{\H_K}}_{=\HSIC(\P_{(n)})}\big|\\&\jover b = \big|\HSIC(\P_0)-\HSIC(\P_{(n)})\big| \jover c = 2Bn^{-1/2}\jover c >0,
\end{align}
where (a) is by the reverse triangle inequality, (b) holds by the definition of the centered cross-covariance operator [\eqref{eq:cov-op-definition}], and (c) follows from \eqref{eq:bound_hsic} and \eqref{eq:HSIC>2s}. We conclude as in the proof of Theorem~\ref{th:hsic} to obtain the statement.

\subsection{Proof of Theorem~\ref{th:ksd}}

Observe that, for a $\P'_0\in\mathcal T$ defined as in Assumption~\ref{as:ksd} (written as $\P_0 \in \mathcal T$ there), we have
\begin{align}
    \MoveEqLeft\inf_{\hat{F}_n} \sup_{\P_0 \in \mathcal T}\sup_{\P\in \mP(K_{\P_0};\X)} \P^n\Big(\left|\KSD(\P_0,\P)-\hat F_n\right|\geq C\Big)\\
    &\jover{a}{\geq} \inf_{\hat{F}_n} \sup_{\P_0\in \left\{\P'_0\right\}}\sup_{\P\in \mP\left(K_{\P_0};\X\right)} \P^n\Big(\left|\KSD(\P_0,\P)-\hat F_n\right|\geq C\Big)
    \\&\jover{b}{=}  \inf_{\hat{F}_n} \sup_{\P\in \mP\left(K_{\P'_0};\X\right)} \P^n\Big(\left|\KSD(\P'_0,\P)-\hat F_n\right|\geq C\Big), \label{eq:ksd_reduction}
\end{align}
where (a) comes by the fact that $\{\P'_0\}\subseteq \mathcal T$ and (b) follows by noting that the supremum of a singleton is attained at its element.  In the following, we relabel $\P_0'$ as $\P_0$; in other words, we write $\P_0 = \P'_0$.

To bring ourselves into the setting of Theorem~\ref{th:le-cam}, for any fixed $n\in \Np$, set $\Y \coloneq \X^n$, $\Theta \coloneq \{\theta_\P \coloneq \KSD(\P_0,\P)\,:\, \P \in \mP(K_{\P_0};\X)\}$, $\mathcal P_{\Theta} \coloneq \{\P^n\,:\, \P\in  \mP(K_{\P_0};\X)\} = \{\P^n\,:\, \theta_\P \in \Theta\}$, and $d(x,y) \coloneq |x-y|$ ($x,y\in \R$). Let us define $F: \mP(K_{\P_0};\X) \to \R$ by $\P\mapsto \KSD(\P_0,\P)$, and let $\hat{F}_n$ denote the corresponding estimator based on $n$ samples. We construct $(\P_{\theta_{0}(n)},\P_{\theta_{1}(n)})$ for fixed $n$, where $\P_{\theta_{0}(n)} \coloneq \P_{\theta_0}$ with $\theta_0 \coloneq \theta_{\P_0}$, $\P_{\theta_1(n)} \coloneq \P_{\theta_n}$ with $\theta_n \coloneq \theta_{\P_{(n)}}$ and $\P_{(n)}\in\mP(K_{\P_0};\X)$ specified below in \eqref{eq:defq1}. With these notations at hand, $d(\theta_0(n),\theta_1(n))=|\KSD(\P_0,\P_0)-\KSD(\P_0,\P_{(n)})| = |0- \KSD(\P_0,\P_{(n)})| = \KSD(\P_0,\P_{(n)})$.

Next, we present the \tb{construction of the adversarial sequence $\P_{(n)}$}.
Let $\varphi\in \mathcal C_b(\X)$ be as constructed in Lemma~\ref{lem:phi-exists} [valid by Assumption~\ref{as:ksd}(ii)], that is, (i) satisfying $\E_{\P_0}[\varphi(X)] = 0$ and (ii) guaranteeing that there exists $A'\in\mathcal B(\X)$ with positive $\P_0$-measure such that $\varphi(x) \neq 0$ for all $x\in A'$. We construct $\P_{(n)}$ as a  perturbation of $\P_0$ taking the form
\begin{align}
\P_{(n)}(A) = \int_A1 + \varepsilon_n\varphi(x)\d \P_0(x) \text{ for any } A \in \mathcal B(\X),
\label{eq:defq1}
\end{align}
with $\varepsilon_n = c n^{-1/2}$ and $c>0$; we also note that $\P_{(n)}\neq \P_0$ (by Lemma~\ref{lem:p_n-neq-p_0}), and that there exists an $n_{0,1}\in \Np$ such that $\P_{(n)}\in\mP(K_{\P_0};\X)$ for all $n\geq n_{0,1}$ [by Lemma~\ref{lem:p_n-and-q_m-inP}; valid as $\P_0\in\mP(K_{\P_0};\X)$ by \eqref{eq:Ep_0K_0=0}].

Having defined $\P_{(n)}$, we continue with the \tb{control of the KSD value $\KSD\left(\P_0,\P_{(n)}\right)$}:
\begin{align}
    \KSD(\P_0,\P_{(n)}) &\overset{\eqref{eq:ksd_def}}{=} \big\|\E_{\P_{(n)}}[\sfm(X)]\big\|_{\H} \jover{a}{=} \Big\|\E_{\P_0}\Big[\sfm(X)\big(1 + \varepsilon_n\varphi(X)\big)\Big]\Big\|_\H\\
    &\,\jover{b}{=} \big\|\underbrace{\E_{\P_0}[\sfm(X)}_{=0 \impliedby \eqref{eq:sfm=0}}] + \varepsilon_n\E_{\P_0}[\varphi(X)\sfm(X)]\big\|_\H \\&\jover{c}{=} \varepsilon_n\underbrace{\big\|\E_{\P_0}[\varphi(X)\sfm(X)]\big\|_\H}_{\eqcolon C_{\varphi}} \jover{d}{>} 0, \label{eq:ksd_control}
\end{align}
where in (a) we used the definition of $\P_{(n)}$ and the property of the Radon-Nikodym derivative, (b) holds by the linearity of the expectation, (c) is implied by the positive homogeneity of norms combined with $\varepsilon_n>0$,  and (d) follows from the fact that $\varepsilon_n > 0$ and that by $\P_{(n)} \neq \P_0$ we have $\KSD(\P_0,\P_{(n)}) > 0$ by the validity of $\KSD$ imposed in Assumption~\ref{as:ksd}(i). Hence,
\begin{align}
\KSD\left(\P_0,\P_{(n)}\right) & = \varepsilon_n C_\varphi \stackrel{\text{(a)}}{=} \Theta\!\left(n^{-1/2}\right), \label{eq:KSD-bound}
\end{align}
where (a) holds by $\varepsilon_n = c n^{-1/2}$ ($c>0$) and  $C_\varphi > 0$.

We proceed by \tb{controlling the KL divergence $\KL\left(\P_{(n)} \| \P_0\right)$}:
Recall that $\varphi \in \mC_b(\X)$, $\E_{\P_0}[\varphi(X)] = 0$, and $\varphi(x) \neq 0$ for all $x$ in some set $A'\in \mathcal B(\X)$ with positive $\P_0$-measure. Then,
\begin{align}
    \KL(\P_{(n)}^n\|\P^n_0) \jover{a}{=}n\KL(\P_{(n)}\|\P_0) \jover {b}{\leq} n\varepsilon_n^2 \alpha_0\jover {c}{=} n\left(cn^{-1/2}\right)^2 \alpha_0 \jover {d}{\eqcolon} \alpha < \infty,\label{eq:ksd_kl_bound}
\end{align}
where Lemma~\ref{lem:tsy-kl-prod-measure} gives (a). Notice that by Lemma~\ref{lem:kl-bound-perturb} there exist $\alpha_0>0$ and $n_{0,2}\in\Np$ such that, for all $n\geq n_{0,2}$, (b) holds. (c) follows from the definition of $\varepsilon_n$, (d) comes by simplification and by introducing $\alpha \coloneq c^2\alpha_0 >0$ where the positivity follows as $c>0$ and $\alpha_0>0$.

Hence, taking $n_0 = \max(n_{0,1},n_{0,2})$, we invoke Theorem~\ref{th:le-cam} [valid by \eqref{eq:ksd_control} and \eqref{eq:ksd_kl_bound}] to obtain that, for all $n\geq n_0$,
\begin{align}
    \inf_{\hat{F}_n} \sup_{\P\in \mP(K_{\P_0};\X)} \P^n\Big(\left|\KSD(\P_0,\P)-\hat F_n\right|\geq C\Big) \geq f(\alpha) > 0, \label{eq:ksd_rhs_bound}
\end{align}
with $f(\alpha)=\max\{\exp(-\alpha)/4,(1-\sqrt{\alpha/2})\}$. Combining \eqref{eq:ksd_reduction} and \eqref{eq:ksd_rhs_bound}, while recalling that we took $\P_0 = \P_0'$ in \eqref{eq:ksd_rhs_bound}, proves the statement.

\section{Auxiliary Results}\label{sec:aux}

In this section we introduce some auxiliary results, used in the proofs of the main results.
Lemma~\ref{lem:phi-not-a.s.c.} provides conditions for a function to not be a.s.\ constant. Lemma~\ref{lem:p_n-and-q_m-inM} shows that perturbed probability measures are probability measures for $n\in\Np$ large enough; Lemma~\ref{lem:p_n-and-q_m-inP} extends this result to $\mP(K;\X)$. Lemma~\ref{cor:p_n-neq-q_m} ensures that different types of perturbed measures are not equal, Lemma~\ref{lem:mu_k-perturbed} shows that the mean embedding of perturbed measures enjoys a simple form, Lemma~\ref{lem:kl-bound-perturb} bounds the KL divergence between perturbed and unperturbed measures, and Lemma~\ref{lem:open-A} collects a key property of the perturbation functions used.  Lemma~\ref{lem:phi-exists-X1xX2} is about the existence of perturbation functions in product spaces, while Lemma~\ref{lem:marginals_pn} makes the marginals of perturbed product measures explicit.

\begin{lemmaA}[Conditions for a function to not be $\P$-a.s.\ constant]\label{lem:phi-not-a.s.c.}
    Let $(\X,\tau_{\X})$ be a topological space and $f\in\mC(\X)$ such that there exist $x \ne y\in\X$ for which $f(x)\neq f(y)$. Assume that $\P\in\M_1^+(\X)$ has full support. Then, $f$ is not $\P$-a.s.\ constant.
\end{lemmaA}
\begin{proof}
   By definition, $f$ is $\P$-a.s.\ constant iff $\P\big(\{x\in \X : f(x) = c \}\big) = \P\big(f^{-1}(C)\big)=1$ for some $c\in\R$ with $C=\{c\}$. We prove that this does not hold for any $C=\{c\}$ with $c\in\R$, showing the statement. Taking an arbitrary $C=\{c\}$, with $c\in\R$, $f^{-1}\big(\R\setminus C\big)\subseteq \X$ is open (as $f$ is continuous and $\R\setminus C$ is open) and non-empty (as there exist $x\neq y\in\X$ with different images); therefore $0 < \P\Big( f^{-1}\big(\R\setminus C\big)\Big)$ as $\supp(\P)=\X$.\footnote{Note that if $A$ is a non-empty open  set and $\supp(\P)=\X$, then $\P(A)>0$. Indeed, we argue by contradiction: we assume that $\P(A)=0$. This implies that $A\jover a \subseteq \X\setminus \supp(\P) \jover b= \emptyset$, where (a) comes by the definition of support, and (b) holds as $\supp(\P)=\X$. Therefore, as $A$ is non-empty, we reached a contradiction.}
   Then,\begin{align}
      0 &< \P\Big( f^{-1}\big(\R\setminus C\big)\Big) \jover a =  \P\Big( f^{-1}(\R)\setminus f^{-1}(C)\Big) = \P\Big(\X\setminus f^{-1}(C)\Big)
      \jover b = \P(\X) - \P\Big(\X \cap f^{-1}(C)\Big) \notag\\
      &\jover c = 1 - \P\Big(f^{-1}(C)\Big) \implies \P\Big(f^{-1}(C)\Big) < 1,
   \end{align}
    where (a) follows from the fact that the pre-image of the set difference is the set difference of pre-images [$f^{-1}(A\setminus B) = f^{-1}(A) \setminus f^{-1}(B)$], (b) is implied  by using $\P(A\setminus B) = \P(A)-\P(A\cap B)$ and (c) comes from the fact that $\P(\X)=1$ and $f^{-1}\big(C\big)\subseteq \X$.
\end{proof}

The following lemma extends an intermediate step in \citet[Section~A.5]{cribeiro26minimax} from $b_n > 0$ ($\epsilon_n > 0$ therein) to $b_n \in \R$.

\begin{lemmaA}[Perturbed measures are in $\M_1^+(\X)$]\label{lem:p_n-and-q_m-inM}
     Let $(\X,\tau_\X)$ be a topological space, $\P_0 \in \mathcal M_1^+(\X)$, and $\varphi \in \mC_b(\X)$ such that $\E_{\P_0}[\varphi(X)]=0$ and $\varphi \neq 0$ on a set of positive $\P_0$-measure. Further, consider the sequence of measures $(\P_{(n)})_{n\in\Np}$, each defined as
     \begin{align}
     \P_{(n)}(A) = \int_A 1+b_n\varphi(x) \d\P_0(x)
     \end{align}
for $A\in \mathcal B(\X)$ with $b_n\in\R$ and $b_n\xrightarrow{n\to \infty}0$. Then, there exists an $n_0\in\Np$ such that for all $n\geq n_{0}$, one has that (i) $1+b_n\varphi(x)\ge 0$ for all $x\in \X$ and (ii) $\P_{(n)}\in\M_1^+(\X)$.
\end{lemmaA}
\begin{proof}
    We need to prove the existence of $n_0\in\Np$ such that for all $n\geq n_0$, one has $\P_{(n)}\in\M_1^+(\X)$; in other words $\P_{(n)}\geq 0$ (which will follow from the intermediate result $1+b_n \varphi\ge 0$) and $\P_{(n)}(\X)=1$.\begin{enumerate}
        \item $\P_{(n)}\geq 0$ for $n\geq n_0$: By the definition of $\P_{(n)}$, it suffices to show that for $n$ large enough
        \begin{align}
            1+b_n\varphi(x) \geq 0\iff -b_n\varphi(x) \leq 1,
        \end{align}
        for all $x\in\X$. Note that $-b_n\varphi(x) \leq 1$ is implied if  $|b_n\varphi(x)|\leq 1$. Since $|\varphi|\in\mC_b(\X)$ [as $\varphi\in\mC_b(\X)$], we have that $|b_n\varphi(x)|\leq |b_n|U\xrightarrow{n\to\infty}0$ using that $b_n \xrightarrow{n\to \infty}0$, with $U\coloneq\sup_{x\in\X}|\varphi(x)|<\infty$. In particular, this implies the existence of an $n_0\in\Np$ such that for all $n\geq n_0$ and for all $x \in \X$, one has $|b_n\varphi(x)|\leq|b_n|U\leq 1$.
        \item $\P_{(n)}(\X) = 1$: Consider
        \begin{align}
            \P_{(n)}(\X) = \int_\X 1+b_n\varphi(x)\d\P_0(x) \jover{a}= \underbrace{\int_\X 1\ \d\P_0(x)}_{=\P_0(\X)} + b_n\underbrace{\int_\X\varphi(x)\d\P_0(x)}_{=\E_{\P_0}[\varphi(X)] \jover{b}{=} 0} \jover{c}= 1,
            \end{align}
            where (a) comes from the linearity of the integral, (b) follows from the assumption imposed on $\varphi$, and (c) holds by the fact that $\P_0\in \M_1^+(\X)$.\qedhere
    \end{enumerate}
\end{proof}

\begin{lemmaA}[Perturbed measures are in $\mP(K;\X)$]\label{lem:p_n-and-q_m-inP}
Let $K:\X^2\to \R$ be a kernel. Assume that  there exists $\P_0\in \mP(K;\X)$ and $\varphi\in\mC_b(\X)$ such that $\E_{\P_0}[\varphi(X)]=0$ and $\varphi(x)\neq 0$ on a set of positive $\P_0$-measure. Consider the sequences of measures $(\P_{(n)})_{n\in\Np}$ and  $(\Q_{(m)})_{m\in\Np}$ on $\big(\X,\mathcal B(\X)\big)$ defined as \begin{align}
    \P_{(n)}(A) = \int_{A} 1+\varepsilon_n\varphi(x)\ \d\P_0(x) \quad \text{ and } \quad \Q_{(m)}(A) = \int_{A} 1- \nu_m\varphi(x)\ \d\P_0(x),
\end{align}
with $\varepsilon_n=c_1n^{-1/2}>0$ and $\nu_m=c_2m^{-1/2}>0$ for some constants $c_1,c_2>0$ and for all $n,m\in\Np$. Then there exists $n_0\in\Np$ such that for all $n\geq n_{0}$ one has $\P_{(n)}\in\mP(K;\X)$ and there exists $m_{0}\in \Np$ such that for all $m\geq m_{0}$ it holds that $\Q_{(m)}\in\mP(K;\X)$.
\end{lemmaA}
\begin{proof}
Recalling the definition of $\mP(K;\X)$, we need to prove that $\P_{(n)},\Q_{(m)}\in\M_1^+(\X)$, and $\int_\X \norm{K(\cdot,x)}{\H_K}\ \d \P_{(n)}(x)<\infty$ and $\int_\X \norm{K(\cdot,x)}{\H_K}\ \d \Q_{(m)}(x)<\infty$ for $n$ and $m$ large enough. We show these properties separately.
\begin{itemize}
     \item \tb{$\P_{(n)},\Q_{(m)} \in\M_1^+(\X)$ for $n\geq n_0$ and $m\geq m_0$}: By Lemma~\ref{lem:p_n-and-q_m-inM} (invoked independently for both $\P_{(n)}$ and $\Q_{(m)}$), there exists $n_0,m_0\in\Np$ such that, for all $n\geq n_0$ and $m\geq m_0$, $\P_{(n)},\Q_{(m)}\in\M_1^+(\X)$.
     \item \tb{$\int_\X \norm{K(\cdot,x)}{\H_K} \d \P_{(n)}(x)<\infty$ for $n\ge n_{0}$, $\int_\X \norm{K(\cdot,x)}{\H_K} \d \Q_{(m)}(x)<\infty$ for $m\ge m_{0}$}:                 Indeed, consider \begin{align}
                    \int_\X \norm{K(\cdot,x)}{\H_K}\ \d\P_{(n)}(x) &\jover a = \int_{\X} \big(1+\varepsilon_n \varphi(x)\big) \norm{K(\cdot,x)}{\H_K}\d \P_0(x)\\
                    &\jover {b} \leq \int_\X \big(1+\varepsilon_nU\big)\norm{K(\cdot,x)}{\H_K} \ \d \P_0(x) \\
                    & \jover c = (1+\varepsilon_n U)\int_\X \norm{K(\cdot,x)}{\H_K}\ \d \P_0(x) \jover d < \infty,
                \end{align}
                where (a) stems from the fact that $\frac{\d \P_{(n)}}{\d\P_0} = 1+\varepsilon_n\varphi$ and the properties of the Radon-Nikodym derivative. Observe that
                \begin{align}
                &\big(1+\varepsilon_n\varphi(x)\big) \norm{K(\cdot,x)}{\H_K} \leq (1+\varepsilon_n U)\norm{K(\cdot,x)}{\H_K},\label{eq:bound_epsilon} \\
                \end{align}
                where $U=\sup_{x\in\X}\varphi(x)<\infty$ [the finiteness of $U$ is guaranteed by $\varphi \in \mC_b(\X)$].
                Then, \eqref{eq:bound_epsilon}, together with the monotonicity of the integral, yield (b). Finally, the homogeneity of the integral gives (c), while the fact that $\P_0\in\mP(K;\X)$ yields (d). The statement regarding $\Q_{(m)}$ follows by substituting $(\varepsilon_n,\varphi)$ by $(\nu_m,-\varphi)$ while noting that (i) $-\varphi\in\mC_b(\X)$ [as $\varphi\in\mC_b(\X)$], (ii) $\E_{\P_0}[-\varphi(\X)]=0$ (by linearity of the expectation), and (iii) $\{x\in\X:~\varphi(x)\neq 0\} = \{x\in\X:~-\varphi(x)\neq 0\}$.\qedhere

        \end{itemize}
\end{proof}

\begin{lemmaA}[Perturbed measures are distinct among themselves]\label{cor:p_n-neq-q_m}
Let $(\X,\tau_\X)$ be a topological space, $\P_0 \in \mathcal M_1^+(\X)$, $\varphi\in\mathcal C_b(\X)$ such that $\varphi \neq 0$ on a set of positive $\P_0$ measure. With $\varepsilon >0$
and $\nu > 0$, define the measures $\P$ and $\Q$ on $(\X,\mathcal B(\X))$ as $\P(A) = \int_{A}1+\varepsilon\varphi(x)\d \P_0(x)$ and $\Q(A) = \int_{A}1-\nu\varphi(x)\d \P_0(x)$, respectively. Then $\Q\neq \P_0$ and $\Q\neq\P$.
\end{lemmaA}
\begin{proof}
    We prove both claims separately.
    \begin{itemize}
        \item \textbf{Case of $\Q\neq \P_0$:} We argue by contradiction, that is, we assume that $\Q = \P_0$.
            Then
                \begin{align}
                    \Q = \P_0
                    &\jover{a}{\implies} \frac{\d \P_0}{\d \P_0}(x) = \frac{\d \Q}{\d \P_0}(x) \;\P_0\text{-a.s.}
                     \\&\jover{b}{\implies} 1 = 1-\nu \varphi(x) \;\P_0\text{-a.s.} \\
                    &\jover{c}{\implies} \varphi(x) = 0 \;\P_0\text{-a.s.},
                \end{align}
                where (a) comes by the almost-sure uniqueness of the Radon-Nikodym derivative, (b) follows by the definition of $\Q$, and we used that $\nu > 0$ in (c). This contradicts the assumed properties of  $\varphi$, concluding the proof.
        \item \textbf{Case of $\Q\neq\P$:} Note that, by assumption, there exists a set $A\in\mathcal B(\X)$ such that $\P_0(A)>0$ and $\varphi(x) \neq 0$ for all $x\in A$. Again arguing by contradiction, we assume that $\Q= \P$.
         Then
         \begin{align}
         \Q = \P  &\jover{a}{\implies} \frac{\d \P}{\d \P_0}(x) = \frac{\d \Q}{\d \P_0}(x) \text{ for all } x \in E,\\
         &\jover{b}{\implies} 1 + \varepsilon\varphi(x) = 1-\nu \varphi(x)\text{ for all } x \in E,\\
          &\implies \varepsilon\varphi(x) = -\nu \varphi(x)  \text{ for all } x \in E, \\
          &\jover{c}{\implies} \varepsilon\varphi(x) = -\nu\varphi(x) \text{ for all }x\in E \cap A,\\
          &\jover{d}{\implies} \varepsilon=-\nu \jover{e}{\implies} \varepsilon<0,
         \end{align}
          where (a) comes by the almost-sure uniqueness of the Radon-Nikodym derivative with equality holding on a measurable set $E$ with $\P_0(E)=1$, (b) follows by the definitions of $\P$ and $\Q$, in (c) the equality was restricted from $E$ to $E \cap A$, where $\P_0(E\cap A)>0$. Indeed, the positivity holds as
          \begin{align}
          A &\jover{f}{=} (A \setminus E) \cup (A \cap E) \Rightarrow\\
          \P_0(A) &\jover{g}{=} \P_0(A \setminus E) + \P_0(A \cap E)\\
          \P_0(A) &\jover{h}{=} \P_0(A \cap E)
          \end{align}
          where $A$ was partitioned in (f), (g) follows from the additivity of $\P_0$. $\P_0(E)=1$ means that $\P_0(E^c)=0$, where $(\cdot)^c$ denotes the set-complement. Hence, combining the fact that $A \setminus E \subseteq E^c$ with the non-negativity and the monotonicity of $\P_0$ gives that $0\le \P_0(A \setminus E)\le \P_0(E^c) = 0$, showing that $\P_0(A \setminus E) = 0$ leading to (h); this implies the claimed positivity by using that $\P_0(A)>0$. (d) holds since $\varphi(x) \neq 0$ for all $x\in E\cap A \subseteq A$, and (e) is implied by $\nu>0$. As $\varepsilon$ was assumed to be positive, we reached a contradiction.
         \qedhere
            \end{itemize}
\end{proof}

\begin{lemmaA}[Mean embedding of perturbed probability measures] \label{lem:mu_k-perturbed}
Let $K:\X^2\to \R$ be a kernel. Consider two probability measures $\P,\Q\in\mP(K;\X)$ such that $\frac{\d \P}{\d \Q}=1 + bf$ with $f\in\mC(\X)$ and $b\in\R$. Then
    \begin{align}
        \mu_K(\P) &= \mu_K(\Q) + b\ \E_\Q\big[K(\cdot,X)f(X)\big].
    \end{align}
\end{lemmaA}
\begin{proof}
    By the definition of the mean embedding
    \begin{align}
        \mu_K(\P) &= \E_\P \big[K(\cdot,X)\big] \jover a = \E_\Q\Big[ K(\cdot,X)\big(1+bf(X)\big)\Big]  \jover b = \E_\Q\big[ K(\cdot,X)\big] + b\ \E_\Q\big[K(\cdot,X)f(X)\big]
        \\ &\jover c =  \mu_K(\Q) + b\ \E_\Q\big[ K(\cdot,X)f(X)\big],
    \end{align}
    where (a) comes from the fact that $\frac{\d \P}{\d \Q}(x) = 1 + b f(x)$ by assumption and the properties of the Radon-Nikodym derivative, (b) holds by the linearity of the expectation, and (c) follows from the definition of the mean embedding.
\end{proof}

\begin{lemmaA}[Bound on the KL divergence for perturbed probability measures]\label{lem:kl-bound-perturb}
    Let $\P_0\in\M_1^+(\X)$ and $\varphi\in\mC_b(\X)$ such that $\E_{\P_0}[\varphi(X)]=0$ and $\varphi \neq 0$ on a set of positive $\P_0$-measure. Let $\P_{(n)}\in\M_1^+(\X)$ such that $\frac{\d \P_{(n)}}{\d \P_0} = 1+b_n\varphi$ with $b_n\in\R$ and $b_n\xrightarrow{n\to\infty}0$. Then, there exist $\alpha>0$ and $n_0\in\Np$ such that for all $n\geq n_0$ one has  $\KL(\P_{(n)}\|\P_0) \leq b^2_n \alpha$.
\end{lemmaA}
\begin{proof}
    Indeed, observe that
    \begin{align}
        \KL(\P_{(n)}\|\P_0) &\jover a = \E_{\P_{(n)}} \left[ \ln\left(\frac{\d\P_{(n)}}{\d\P_0}(X)\right) \right] \jover b = \E_{\P_0} \left[ \frac{\d\P_{(n)}}{\d\P_0}(X)\ln\left(\frac{\d\P_{(n)}}{\d\P_0}(X)\right) \right]\\
        &
        \jover c=\E_{\P_0} \Big[\big(1+b_n\varphi(X)\big)\ln\big(1+b_n\varphi(X)\big)\Big],
    \end{align}
    where (a) comes from the definition of the KL divergence, (b) is implied by the properties of Radon-Nikodym derivatives, and (c) holds by $\tfrac{\d\P_{(n)}}{\d\P_0} = 1 + b_n\varphi$.
    Recall that for any $z>-1$, $\ln(1+z)\leq z$. To be able to apply this inequality with $z = b_n\varphi(x)$, we need that $b_n\varphi(x)>-1\iff -b_n\varphi(x)<1$. In particular, if $|b_n\varphi(x)|<1$ one has that $-b_n\varphi(x) < 1$. Since $|\varphi|\in\mC_b(\X)$ [as $\varphi\in\mC_b(\X)$], we have that $|b_n\varphi(x)|\leq|b_n|U\xrightarrow{n\to\infty}0$ using the fact that $b_n \xrightarrow{n\to \infty} 0$, with $U\coloneq \sup_{x\in\X}|\varphi(x)|<\infty$. Hence, there exists an $n_0\in\Np$ such that for all $n\geq n_0$ and for all $x \in \X$ one has $|b_n\varphi(x)|\leq |b_n|U<1$.
    Therefore, taking $n\geq n_0$,
    \begin{align}
        \E_{\P_0} \Big[\big(1+b_n\varphi({X})\big)\underbrace{\ln\big(1+b_n\varphi({X})\big)}_{\leq b_n\varphi({X})}\Big]
        &\jover a \leq \E_{\P_0} \Big[\big(1+b_n\varphi({X})\big)\ b_n\varphi({X})\Big]
      \\&\jover b =  b_n\E_{\P_0} \big[\varphi({X})\big] +  b_n^2\ \underbrace{\E_{\P_0} \big[\varphi^2({X})\big]}_{\eqcolon \alpha<\infty}
        \jover c = b_n^2\alpha \jover d <\infty,
    \end{align}
    where (a) comes by the monotonicity of the integral together with the aforementioned bound on the logarithm and that $1+b_n \varphi(x)>0$ for all $x\in \X$, (b) follows by distributivity and the linearity of the expectation. The fact that $\E_{\P_0}[\varphi(X)] =0$ (by assumption) yields (c); the finiteness of $b_n$ and $\alpha$ [the latter following from $\varphi^2 \in \mathcal C_b(\X)$, implied by $\varphi \in \mathcal C_b(\X)$] yields (d).
    The property that $\alpha>0$ follows from the fact that $\varphi^2\ge 0$ is positive on a Borel set with positive $\P_0$-measure. This proves the statement.
\end{proof}

\begin{lemmaA}[Existence of an open set of positive measure for the perturbations]\label{lem:open-A}
    Let $(\X,\tau_\X)$ be a topological space, $\P\in\M_1^+(\X)$, and $\varphi\in\mC(\X)$ such that $\varphi \neq 0$ on a set of positive $\P$-measure. Then, there exists an $A\in\tau_\X$ with positive $\P$-measure, such that $\varphi(x) \neq 0$ for all $x\in A$.
\end{lemmaA}
\begin{proof}
    By assumption, there exists an $A'\in\mathcal B(\X)$ with $\P(A') > 0$ and such that $\varphi(x) \neq 0$ for all $x\in A'$. As $\{0\}\subset \R$ is closed, $\R\setminus\{0\}$ is open; therefore, the pre-image $\varphi^{-1}\big(\R\setminus\{0\}\big)\eqcolon A$ is open (in other words, $A\in \tau_{\X}$; by the continuity of $\varphi$), nonempty (as the existence of $x\in\mathcal X$ such that $\varphi(x)\neq 0$ was assumed). By definition $\varphi(x)\neq 0$ for all $x\in A$. Notice that
    \begin{align}
        A' \jover a \subseteq \varphi^{-1}\big(\varphi(A')\big)\jover b \subseteq A,
    \end{align}
    where (a) follows from the properties of the pre-image\footnote{Note that the image and pre-image of a map $f:\X\to\Y$ are such that $A \subseteq f^{-1}\big(f(A)\big)\subseteq \X$ \citep[Proposition~3.14]{sutherland09introduction}. Further, for $A_1\subseteq A_2\subseteq \X$ and $B_1\subseteq B_2\subseteq \Y$, we have $f(A_1)\subseteq f(A_2) \subseteq f(\X)$ and $f^{-1}(B_1)\subseteq f^{-1}(B_2) \subseteq f^{-1}(\Y)$ \citep[Exercise~3.1]{sutherland09introduction}.} and (b) from the fact that $\varphi(A')\subseteq \R\setminus\{0\}$ by definition.
    Hence, as $A'\subseteq A$ and $\P(A')>0$, we have by the monotonicity of probability measures that $\P(A)\geq \P(A')>0$.
    \end{proof}

\begin{lemmaA}[Existence of perturbation function in a product space]\label{lem:phi-exists-X1xX2}
   For $i\in[d]$, let $(\X_i,\tau_{\X_i})$ be topological spaces, and the pairs $(\P_i,\varphi_i)\in\M_1^+(\X_i)\times\mC_b(\X_i)$
   such that there exists no $c_i\in \R$ for which $\varphi_i=c_i$ holds $\P_i$-a.s. %
   Define $\X \coloneq \times_{i=1}^d\X_i$ and $\P_0 \coloneq \otimes_{i=1}^d \P_i \in\mathcal M_1^+(\mathcal X)$. Then, for each $i\in[d]$, there exists a $\varphi'_i \in \mC_b(\X_i)$ such that $\E_{\P_i}[\varphi'_i(X_i)]= 0$. Moreover, with $\varphi = \prod_{i=1}^d \varphi'_i \in\mC_b(\X)$, it holds that $\E_{\P_0}[\varphi(X)]=0$ and that $\varphi\neq 0$ on a set of positive $\P_0$-measure.
\end{lemmaA}
\begin{proof}
By our imposed assumptions, we can invoke Lemma~\ref{lem:phi-exists} and Lemma~\ref{lem:open-A} for each $i\in[d]$ to guarantee the existence of $\varphi'_i\in\mC_b(\X_i)$ such that (i) $\varphi'_i(x_i)\neq 0$ for all $x_i$ in some $A_i\in\tau_{\X_i}$ of positive $\P_i$-measure and all $i\in[d]$, and (ii)
\begin{align}\label{eq:exp_phi_i=0}
    \E_{\P_i}[\varphi'_i(X_i)] = 0\text{ for all } i\in[d].
\end{align}
We prove that $\varphi=\prod_{i=1}^d\varphi'_i\in\mC_b(\X)$ and it satisfies the two stated additional properties in what follows

\begin{itemize}
\item \tb{$\varphi=\prod_{i=1}^d\varphi'_i\in\mC_b(\X)$.} Indeed, the boundedness of $\varphi$ holds as
\begin{align}
\sup_{x \in \X} |\varphi(x)| &= \sup_{x \in \X} \left|\prod_{i=1}^d\varphi'_i(x_i)\right| \le \prod_{i=1}^d \underbrace{\sup_{x_i \in \X_i} |\varphi'_i(x_i)|}_{<\infty\, \Leftarrow\, \varphi_i' \text{ is bounded}} < \infty.
\end{align}
Note that $\varphi$ can be written as the composition $(x_1,\dots,x_d)\mapsto \big(\varphi_1(x_1),\dots,\varphi_d(x_d)\big) \mapsto \prod_{i=1}^d \varphi_i(x_i) = \varphi(x_1,\dots,x_d)$. Both maps are continuous, for the former see \citep[Exercise 11, Chapter 2]{munkres00topology}; $\varphi$ is continuous as it is the composition of continuous functions.

\item \tb{$\E_{\P_0}[\varphi(X)]=0$.}~Indeed,
\begin{align}\label{eq:exp_pk_phi=0}
    \E_{\P_0}[\varphi(X)] \jover{a}{=} \E_{\P_1}\Bigg[\cdots\E_{\P_d}\Bigg[\prod_{i=1}^d \varphi'_{i}(X_{i})\Bigg]\Bigg] \jover{b}{=} \prod_{i=1}^d \underbrace{\E_{\P_i}[\varphi'_i(X_i)]}_{\overset{\eqref{eq:exp_phi_i=0}}{=}0} = 0,
\end{align}
where (a) comes from the definition of $\varphi$ and the properties of the product measure $\P_0=\otimes_{i=1}^d\P_i$, and (b) holds by the homogeneity of the expectation.

\item \tb{$\varphi(x)\neq0$ on a set of positive $\P_0$-measure.}~Recall that $A_i\in\tau_{\X_i}$ as defined earlier for all $i\in[d]$ and define $A=\times_{i=1}^d A_i$. Then, we have that $\varphi(x)=\prod_{i=1}^d\varphi'_i(x_i)\neq 0$, for all $ x=(x_i)_{i=1}^d \in A$. Note that $A \subseteq \X $ is an open set in the product topology, and thus, $A\in\mathcal B(\X)$. Further, by the definition of  product measure, $\otimes_{i=1}^d \P_i(A)=\prod_{i=1}^d\underbrace{\P_i(A_i)}_{>0}>0$.\qedhere
\end{itemize}
\end{proof}

\begin{lemmaA}[Marginals of a perturbed measure in a product space]\label{lem:marginals_pn}
    Let $(\X_i,\tau_{\X_i})$ be topological spaces and $\P_i \in\M_1^+(\X_i)$ for all $i\in[d]$, $\P_0 \coloneq \otimes_{i=1}^d \P_i$ and  $\X\coloneq\times_{i=1}^d\X_i$. Let $\varphi \in \mC_b(\X)$ be as in Lemma~\ref{lem:phi-exists-X1xX2} %
    and let $\P\in\M_1^+(\X)$ be a probability measure taking the form
    \begin{align}
    \P(A) &= \int_A 1 + b\varphi(x)~\d\P_0(x)
    \end{align}
    for all $A\in \mathcal{B}(\X)$ with a fixed $b \in \R$. Then, $\marg{\P}{i} = \P_i$, for all $i \in [d]$.
\end{lemmaA}
\begin{proof}
    We will prove the statement for $i = 1$ without loss of generality; the remaining values of $i\in[d]$ can be reduced to this case by changing the order of the integration using Fubini's theorem [Theorem~\ref{th::ex::Fubini}; guaranteed as $\varphi \in \mC_b(\X)$]. First, recall that
    $\varphi = \prod_{i=1}^d\varphi'_i$ and \begin{align}\label{eq:exp_varphi=0_lemma_b6}
        \E_{\P_i}\big[\varphi'_i(X_i)\big] = 0,~\text{for all } i\in[d].
    \end{align}
    Then, for any $A_1\in\mathcal B(\X_1)$,
    \begin{align}
        \marg{\P}{1} (A_1)&\jover{a}{=} \int_{A_1\times\X_2\times\cdots\times\X_d}1\ \d\P(x)
        \jover b = \int_{A_1\times\X_2\times\cdots\times\X_d} 1 + b\varphi(x) \ \d\P_0(x)\\
        &\jover c = \int_{A_1\times\X_2\times\cdots\times\X_d} 1\ \d \P_0(x) + b\int_{A_1\times\X_2\times\cdots\times\X_d} \varphi(x)\ \d \P_0(x) \\
        &\jover d =  \P_1(A_1) + b\int_{A_1\times\X_2\times\cdots\times\X_d} \varphi(x)\ \d \P_0(x)\\
        &\jover e =\P_1(A_1) +  b \int_{A_1}\int_{\X_2}\cdots\int_{\X_d}\prod_{i=1}^d\varphi'_i(x_i)\ \d \P_d(x_d)\cdots\d \P_2(x_2)\d \P_1(x_1)  \\
        &\jover f = \P_1(A_1) +  b \int_{A_1}\varphi'_1(x_1)\d\P_1(x_1)\prod_{i=2}^d \underbrace{\int_{\X_i}\varphi'_i(x_i)\d\P_i(x_i)}_{\eqcolon \E_{\P_i}[\varphi'_i(X_i)]} \jover g =\P_1(A_1) + 0,
    \end{align}
    where (a) comes from the definition of marginals, (b) holds by the fact that $\frac{\d \P}{\d \P_0} = 1 + b\varphi$ and the properties of the Radon-Nikodym derivative, (c) is by the linearity of the integral, (d) holds by
    \begin{align}
   \int_{A_1\times\X_2\times\cdots\times\X_d} 1\ \d \P_0(x) &= \P_0(A_1\times\X_2\times\cdots\times\X_d) = \P_1(A_1) \prod_{i=2}^d \underbrace{\P_i(\X_i)}_{=1} = \P_1(A_1),
    \end{align}
    using the definition of $\P_0$ and that $\P_i \in \M_1^+(\X_i)$. The definition of $\varphi$ gives (e), the homogeneity of the integral yields (f),
    and \eqref{eq:exp_varphi=0_lemma_b6} implies (g).
\end{proof}

\section{External Results}\label{sec:ext}

This section collects the external statements that we use. Lemma~\ref{lem:phi-exists} ensures the existence of a perturbation function on a topological space and Lemma~\ref{lem:p_n-neq-p_0} guarantees that a perturbed measure is distinct from its original measure. Lemma~\ref{lem:tsy-kl-prod-measure} is about the Kullback-Leibler (KL) divergence of product measures. Theorem~\ref{th::ex::Fubini} and Theorem~\ref{th:le-cam} restate the well-known results of Fubini and Le Cam, respectively.

\begin{lemmaA}[Existence of perturbation function; Lemma~B.4, \citealt{cribeiro26minimax}] \label{lem:phi-exists} Let $(\X,\tau_\X)$ be a topological space, $\P_0 \in \mathcal M_1^+(\X)$, and $\varphi_0\in\mathcal C_b(\X)$ such that there exists no $c\in\R$ such that $\varphi_0=c$ holds $\P_0$-a.s. Then there exists $\varphi \in  \mathcal C_{b}(\X)$ such that $\varphi(x) \neq 0$ on a set of positive $\P_0$ measure and $\E_{\P_0}[\varphi(X)] = 0$.
\end{lemmaA}

\begin{lemmaA}[Perturbed measures are distinct; Lemma~B.5, \citealt{cribeiro26minimax}]\label{lem:p_n-neq-p_0}
Let $(\X,\tau_\X)$ be a topological space, $\P_0 \in \mathcal M_1^+(\X)$, $\varphi\in\mathcal C_b(\X)$ such that $\varphi \neq 0$ on a set of positive $\P_0$ measure, and $\varepsilon >0$. Define the measure $\P$ on $(\X,\mathcal B(\X))$ as $\P(A) = \int_A1+\varepsilon\varphi(x)\d \P_0(x)$. Then $\P_0 \neq \P$.
\end{lemmaA}

\begin{lemmaA}[KL divergence of product measures; p.~85, \citealt{tsybakov09introduction}]
\label{lem:tsy-kl-prod-measure}
Let $\P=\otimes_{j=1}^n\P_j$ and $\Q=\otimes_{j=1}^n\Q_j$. Then
$\mathrm{KL}(\P\|\Q) = \sum_{j=1}^n\mathrm{KL}(\P_j\|\Q_j)$.
\end{lemmaA}

The following theorem allows to exchange the order of integration. We recall that $\sigma$-finiteness holds for any probability measure.
\begin{theoremA}[Fubini-Tonelli; Theorem 2.37.b, \citealt{folland99real}]\label{th::ex::Fubini}
    Suppose that $(\mathcal{X},\mathcal{M},\nu_1)$ and $(\mathcal{Y},\mathcal{N},\nu_2)$ are $\sigma$-finite measure spaces. Let $f:  \mathcal{X}\times \mathcal{Y} \to \R$ be a measurable function.
    If $\int_{\X\times\mathcal Y} |f(x,y)| \d (\nu_1\otimes\nu_2)(x,y) < \infty$, then
    \begin{equation}
        \int_{\mathcal{X}\times \mathcal{Y}}f(x,y)\d(\nu_1\otimes\nu_2)(x, y) = \int_\mathcal{X} \left[\int_\mathcal{Y} f(x,y) \d \nu_2(y) \right]\d \nu_1 (x) = \int_\mathcal{Y} \left[\int_\mathcal{X} f(x,y) \d \nu_1(x) \right]\d \nu_2 (y).
    \end{equation}
\end{theoremA}

\begin{theoremA}[Le Cam's method; \citealt{lecam73convergence}; Theorem 2.2, \citealt{tsybakov09introduction}]
\label{th:le-cam}
Let $\mathcal Y$ be a measurable space, $(\Theta,d)$ a semi-metric space, and $\mathcal P_{\Theta} =
\{\P_\theta : \theta \in  \Theta\}$ a class of probability measures
on $\mathcal Y$ indexed by $\Theta$. We
observe data $D \sim \P_{\theta} \in \mathcal P_{ \Theta}$ with some unknown
parameter $\theta$. The goal is to estimate $\theta$. Let $\hat \theta = \hat \theta(D)$ be an estimator of $\theta$ based on~$D$.
Assume that there exist $\theta_1,\theta_2\in \Theta$ such that
$d(\theta_1,\theta_2) \ge 2s > 0$ and $\mathrm{KL}(\P_{\theta_2}||\P_{\theta_1})
\le \alpha <\infty$ for $\alpha > 0$. Then
\begin{align*}
  \inf_{\hat \theta}\sup_{\theta\in \Theta}\P_\theta\big(d\big(\hat \theta,\theta\big) \ge s\big) \ge f(\alpha),
\end{align*}
with $f(\alpha) \coloneq \max \!\big\{\exp(-\alpha) / 4,(1-\sqrt{\alpha/2})\big\}>0$.
\end{theoremA}

\bibliography{BIB/collected,BIB/collected_plus,BIB/publications_own
}

\vfill
\end{document}